\documentclass[journal]{IEEEtran}
\IEEEoverridecommandlockouts
\usepackage{cite}
\usepackage{amsmath,amssymb,amsfonts}
\usepackage{algorithmic}
\usepackage{graphicx}
\usepackage{textcomp}
\usepackage{xcolor}
\def\BibTeX{{\rm B\kern-.05em{\sc i\kern-.025em b}\kern-.08em
    T\kern-.1667em\lower.7ex\hbox{E}\kern-.125emX}}

\usepackage[bookmarks=true,breaklinks,hidelinks]{hyperref}  
\PassOptionsToPackage{hyphens}{url}
\usepackage[hyphens]{url}

\usepackage{algorithm}

\usepackage{balance} 
\usepackage{mathtools}

\usepackage{multirow} \usepackage{hhline} 
\graphicspath{{figures/}} 

\ifCLASSOPTIONcompsoc
\usepackage[caption=false,font=normalsize,labelfont=sf,textfont=sf]{subfig}
\else
\usepackage[caption=false,font=footnotesize]{subfig}
\fi

\newcommand{\Tau}{\mathrm{T}}

\usepackage{array}
\newcommand{\PreserveBackslash}[1]{\let\temp=\\#1\let\\=\temp}
\newcolumntype{C}[1]{>{\PreserveBackslash\centering}p{#1}}
\newcolumntype{R}[1]{>{\PreserveBackslash\raggedleft}p{#1}}
\newcolumntype{L}[1]{>{\PreserveBackslash\raggedright}p{#1}}

\usepackage{amsthm}
\newtheorem{proposition}{Proposition}
\newtheorem{theorem}{Theorem}

\newtheorem{corollary}{Corollary}
\newtheorem{remark}{Remark}

\usepackage[shortlabels]{enumitem}

\newif\iftodo
\todofalse

\iftodo
\newcommand{\todo}[1]{\color{red}#1\color{black}}
\else
\newcommand{\todo}[1]{}
\fi

\newif\ifnote
\notefalse

\ifnote
\newcommand{\note}[1]{{\color{blue}\bf {#1} \color{black}}}
\else
\newcommand{\note}[1]{}
\fi

\newif\ifedit
\editfalse

\ifedit
\newcommand{\edit}[1]{\color{blue}#1\color{black}}
\else
\newcommand{\edit}[1]{\color{black}#1}
\fi

\newif\ifeedit
\eeditfalse

\ifeedit
\newcommand{\eedit}[1]{\color{blue}#1\color{black}}
\else
\newcommand{\eedit}[1]{\color{black}#1}
\fi

\begin{document}

\title{
Beyond Inverted Pendulums: Task-optimal Simple Models of Legged Locomotion
}

\author{Yu-Ming Chen$^{1}$, Jianshu Hu$^{2}$ and Michael Posa$^{1}$ \thanks{$^{1}$The authors are with the General Robotics, Automation, Sensing and Perception (GRASP) Laboratory, University of Pennsylvania, Philadelphia, PA 19104, USA.
        {\tt\small \{yminchen, posa\}@seas.upenn.edu}}\thanks{$^{2}$The author is with the UM-SJTU Joint Institute, Shanghai Jiao Tong University, Shanghai, China.
        {\tt\small hjs1998@sjtu.edu.cn}}}

\maketitle

\setcounter{footnote}{2}

\begin{abstract}

Reduced-order models (ROM) are popular in online motion planning due to their simplicity.
\edit{
A good ROM for control captures critical task-relevant aspects of the full dynamics while remaining low dimensional. }
However, planning within the reduced-order space unavoidably constrains the full model, and hence we sacrifice the full potential of the robot.
In the community of legged locomotion, this has lead to a search for better model extensions, but many of these extensions require human intuition,
and there has not existed a principled way of evaluating the model performance and discovering new models.
In this work, we propose a model optimization algorithm that automatically synthesizes reduced-order models, optimal with respect to a user-specified distribution of tasks and corresponding cost functions.
To demonstrate our work, we optimized models for a bipedal robot Cassie. 
\edit{
We show in simulation that the optimal ROM reduces the cost of Cassie's joint torques by up to 23\% and increases its walking speed by up to 54\%.
We also show hardware result that the real robot walks on flat ground with 10\% lower torque cost.
All videos and code can be found at \\ \href{https://sites.google.com/view/ymchen/research/optimal-rom}{https://sites.google.com/view/ymchen/research/optimal-rom}.

}

\end{abstract}


\section{Introduction}

\begin{figure*}[t!]
 \centering
    \includegraphics[width=1.0\linewidth]{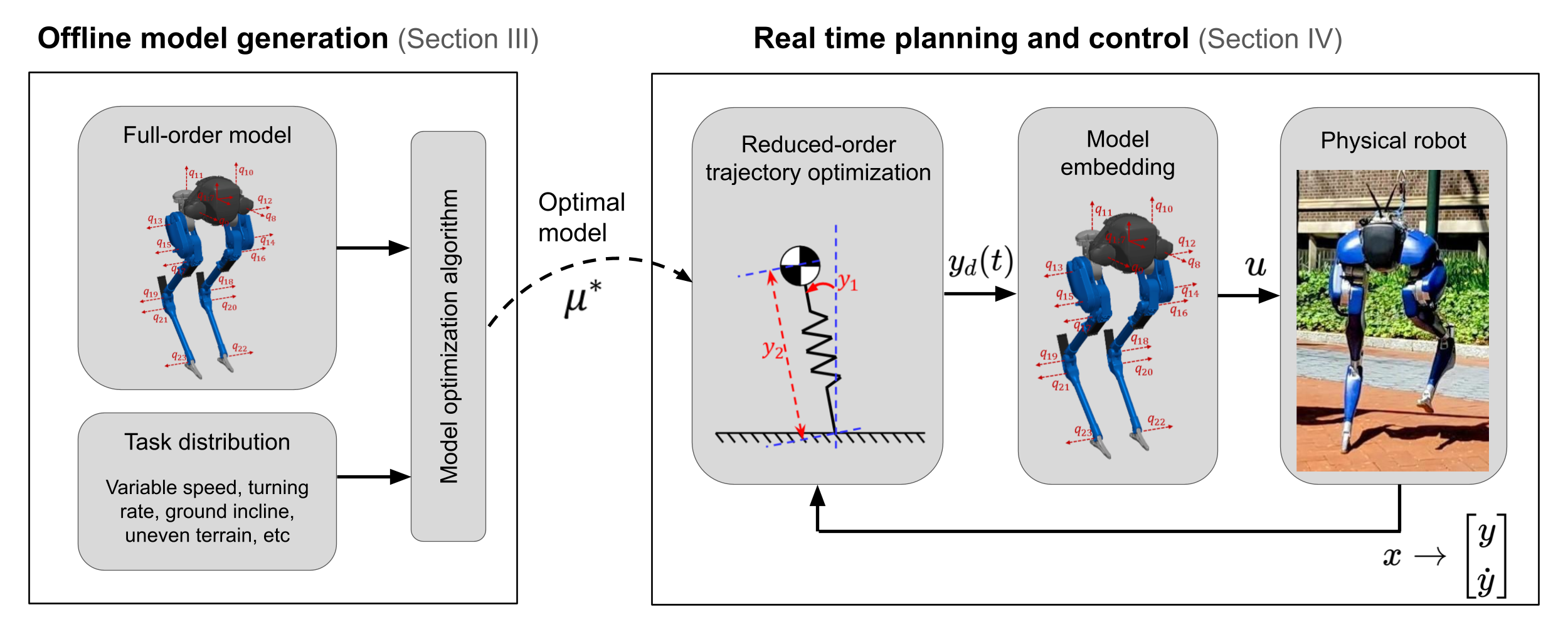}
 
 \vspace{-3mm}
 
 \caption{An outline of the synthesis and deployment of optimal reduced-order models (ROM). 
 Offline, given a full-order model and a distribution of tasks, we optimize a new model that is effective over the task space (Section \ref{sec:approaches}). 
 Online, we generate new plans for the reduced-order model and track these trajectories on the true, full-order system (Section \ref{sec:mpc}).
 This diagram also shows the bipedal robot Cassie (in the rightmost box) and its full model. 
 Cassie has five motors on each leg -- three located at the hip, one at the knee and one at the toe. 
 Additionally, there are 2 leaf springs in each leg, and the spring joints are visualized by $q_{16}$ to $q_{19}$ in the figure.
 The springs are a part of the closed-loop linkages of the legs.
 We model these linkages with distance constraints, so there are no rods visualized in the model. 
 }
\vspace{-2mm}
 \label{fig:outline}
\end{figure*}

State-of-the-art approaches to model-based planning and control of legged locomotion can be categorized into two types \cite{wensing2023optimization}.
One uses the full-order model of a robot, and the other uses a reduced-order model (ROM). 
With the full model, we can leverage our full knowledge about the robot to achieve high performance \cite{westervelt2003hybrid, zhao2017multi, reher2016realizing}. However, this comes with the cost of heavy computation load, and it also poses a challenge in formal analysis,
because modern legged robots 
have many degrees of freedom. To manage this complexity, the community of legged robots has embraced the use of reduced-order models.

Most reduced-order models adopt constraints (assumptions) on the full model dynamics \edit{while capturing the task-relevant part of the full-order dynamics}.
For example, the linear inverted pendulum (LIP) model \cite{Kajita91, Kajita01} assumes that the robot is a point mass that stays in a plane, which greatly reduces energy efficiency and limits the speed and stride length of the robot.
The spring loaded inverted pendulum (SLIP) model \cite{blickhan1989spring} is a point mass model with spring-mass dynamics, which implies zero centroidal angular momentum rate and \edit{ zero ground impacts at foot touchdown event}.
Therefore, when we plan for motions only in the reduced-space, we unavoidably impose limitations on the full dynamics. 
This restricts a complex robot's motion to that of the low-dimensional model and necessarily sacrifices performance of the robot.

The above limitations of the reduced-order models have long been acknowledged by the community, resulting in a wide array of extensions that universally rely on human intuition, and are often in the form of mechanical components (a spring, a damper, a rigid body, the second leg, etc) \cite{xiong2020dynamic, xiong2021global, kasaei2020robust, takenaka2009real, sato2010real, shimmyo2012biped, kasaei2018reliable, faraji20173lp}. 
The success of such model extensions which enable high-performance real time planning on hardware are listed as follows.
Chignoli et al. \cite{chignoli2021humanoid} used a single rigid body model and assumed small body pitch and roll angle, in order to formulate a convex planning problem.
Xiong et al. \cite{xiong2020dynamic} extended LIP with a double-support phase while maintaining the \edit{ zero ground impact }assumption, so the model is still linear and conducive to a LQR controller \cite{shaiju2008formulas, garcia1989model, borrelli2017predictive}.
Gibson et al. \cite{gong2022zero, gibson2022terrain} used the angular-momentum-based LIP which has better prediction accuracy than the traditional LIP.
Dai et al. and Herzog et al. \cite{Dai14, herzog2016structured} combined the centroidal momentum model with full robot configurations, and its real time application in planning was made possible thanks to Boston Dynamics' software engineering \cite{AtlasMPC}.

While some of the above extensions have improved the robot performance, it remains unclear which extension provides more performance improvement than the others, and we do not have a metric to improve the model performance with. 
Additionally, it has been shown that not all model extensions can significantly improve the performance of robots. 
For example, allowing the center of mass height to vary provides limited aid in the task of balancing \cite{Posa17, Koolen16}.
In our work, we aim to automatically discover the most beneficial extension of the reduced-order models by directly optimizing the models given a user-specified objective function and a task distribution.

\subsection{Related Work}

Several researchers have sought to enhance the performance of the reduced-order model by mixing it with a full model in the planning horizon of Model Predictive Control (MPC).
Li et al. \cite{li2021model} divided the horizon into two segments, utilizing a full model for the immediate part and a reduced-order model for the distant part.
\edit{
Subsequently, Khazoom et al. \cite{khazoom2023optimal} systematically determined the optimal scheduling of these two models. }
Norby et al. \cite{norby2022adaptive} blended the full and reduced-order models while adaptively switching between the two.
Our work is different from these existing works in that we directly optimize a reduced-order model instead of reasoning about scheduling two models to improve the overall model performance.

\edit{
Classical approaches to finding a ROM often minimize the error between the ROM and the full model.
Pandala et al. \cite{pandala2022robust} attempted to close this gap implicitly in a learning framework, modeling the difference between the two models as a disturbance to the reduced dynamics.
Our prior work \cite{chen2023integrable} looked for an Integrable Whole-body Orientation model
by minimizing the angular momentum error between the reduced-order and full-order model.
Yamane \cite{yamane2012systematic} linearized the full model and reduced the dimension via principal component analysis (PCA), resulting in a ROM optimally approximating the full dynamics in a neighborhood of the linearization point.
In contrast to these works, 
the approach presented in this paper leverages the fact that control on the full robot can be used to exactly embed low-dimensional models, 
and thus we judge the quality of such a model by a user-specified cost function instead of the modeling error. 
This definition and mechanism for assessing a ROM align with the state-of-the-art methods in model-based planning and control of legged robots -- plan trajectories/inputs in the ROM space first and then track these reduced-order trajectories/inputs on the robot.

}

\vspace{-3mm}

\subsection{Contributions}

The contributions of this paper are:

\begin{enumerate}
\item We propose a bilevel optimization algorithm to automatically synthesize new reduced-order models, embedding high-performance capabilities within low-dimensional representations. (This contribution was presented in the conference form \cite{chen2020optimal} of this paper.) 
\item We improve the model formulation of the prior work \cite{chen2020optimal}, and improve the algorithm efficiency by using the Envelope Theorem to derive the analytical gradient of an optimization problem.
We provide more examples of model optimization, with different sizes of task space and basis functions.
\item We design a real time MPC controller for the optimal ROM, and demonstrate that the optimal model is capable of achieving higher performance in both simulation and hardware experiment of a bipedal robot Cassie. 
\item We evaluate and compare the performances of reduced-order models in both simulation and hardware experiments. 
We analyze the performance gain and discuss the lessons learned in translating the model performance of an open-loop system to a closed-loop system.
\end{enumerate}

\vspace{-3mm}

\subsection{Organization}
The paper is organized as follows.
Section \ref{sec:background} introduces the models of the Cassie robot and the background for this paper.
Section \ref{sec:approaches} introduces our definition of a reduced-order model, formulates the model optimization problem, provides an algorithm that solves the problem, and finally demonstrates model optimization with a few examples.
Section \ref{sec:mpc} introduces an MPC for a specific class of ROMs used in Section \ref{sec:closed_loop_eval}.
Section \ref{sec:closed_loop_eval} compares and analyzes the performance improvement in trajectory optimization, in simulation and in hardware.
Section \ref{sec:mpc_generic_rom} discusses the hybrid nature of legged robot dynamics and introduces the MPC for a general ROM.
Finally, we discuss some of the lessons learned during the journey of realizing better performance on the robot in Section \ref{sec:discussion}, and conclude the paper in Section \ref{sec:conclusion}.
\edit{
The link to all videos and code for the examples is provided in the Abstract. }

\vspace{-1mm}

\section{Background}\label{sec:background}

\subsection{Reduced-order Models of Legged Locomotion}

Modern legged robots like the Agility Robotics Cassie have many degrees of freedom and may incorporate passive dynamic elements such as springs and dampers.
To manage this complexity and simplify the design of planning and control, reduced-order models have been widely adopted in the research community.

One observation, common to many approaches, lies in the relationship between foot placement, ground reaction forces, and the center of mass (CoM) \cite{Full1999}.
While focusing on the CoM neglects the individual robot limbs, controlling the CoM position has proven to be an excellent proxy for the stability of a walking robot.
CoM-based simple models include the LIP \cite{Kajita91,Kajita01}, SLIP \cite{blickhan1989spring}, hopping models \cite{Raibert1984}, inverted pendulums \cite{Garcia1998,Schwab2001}, and others.
Since these models are universally low-dimensional, they have enabled a variety of control synthesis and analysis techniques that would not otherwise be computationally tractable.
For example, numerical methods have been successful at finding robust gaits and control designs \cite{Byl2009,OguzSaglam2014, Kelly2015, Koolen12}, and assessing stability \cite{Pratt06}.

\edit{
Many of the aforementioned reduced-order models feature massless legs, eliminating any foot-ground impact during the swing foot touchdown event.
When dealing with a robot or a model incorporating a foot of non-negligible mass, zero impacts necessitate zero swing foot velocity at touchdown.
This constraint ensures the velocity continuity before and after the touchdown event.
}

A popular approach to using reduced-order models on legged robots is to first plan with the ROM to get desired ROM trajectories and desired foot steps, and then use a low-level controller to track the planned trajectories.
This workflow is depicted in the right half of Fig. \ref{fig:outline}.
As planning speed improves (e.g., through model complexity reduction), 
real time solving of the planning problem with a receding horizon (MPC) becomes feasible.

\vspace{-2mm}

\subsection{Models of Cassie}\label{sec:cassie_model}

The bipedal robot Cassie (Fig. \ref{fig:outline}) is the platform we used to test our model optimization algorithm. 
Here we briefly introduce its model. 
Let the state of Cassie be $x = [q, v] \in \mathbb{R}^{45}$ where $q \in \mathbb{R}^{23}$ and $v \in \mathbb{R}^{22}$ are generalized position and velocity, respectively. 
We note that $q$ and $v$ have different dimensions, because the floating base orientation is expressed via quaternion. 
The conversion between  $\dot{q}$ and $v$ depends only on $q$ \cite{wie1985quaternion}. 
The standard equations of motion are
\vspace{-1mm}
\begin{equation}\label{eq:eom}
M(q) \dot{v} = f_{cg} (q,v) + Bu + J_h(q)^T \lambda + \tau_{app}(q,v) 
\vspace{-1mm}
\end{equation}
where $M$ is the mass matrix which includes the reflected inertia of motors, 
$f_{cg}$ contains the velocity product terms and the gravitational term, 
$B$ is the actuation selection matrix, 
$u$ is the actuator input, 
$J_h$ is the Jacobian of holonomic constraints associated with the constraint forces $\lambda$, 
and $\tau_{app}$ includes the other generalized forces applied on the system such as joint damping forces. \edit{
The forces $\lambda$ contain ground contact forces and constraint forces internal to the four-bar linkages of Cassie.
In simulation, the ground forces are calculated by solving an optimization problem based on the simulator's contact model \cite{Drake2016}.
In trajectory optimization, the forces $\lambda$ are solved simultaneously with $x$ and $u$ while satisfying the dynamics, holonomic and friction cone constraints.
}
Furthermore, we assume the swing foot collision with the ground during walking is perfectly inelastic \edit{in the trajectory optimization}, so the robot dynamics is hybrid.  
Combining the discrete impact dynamics (from foot collision) with Eq. \eqref{eq:eom}, we derive the hybrid equations of motion 
\edit{
\begin{equation}\label{eq:hybrid_eom}
    \begin{cases}
    
\begin{alignedat}{2}
\dot{x} &= f\left(x, u, \lambda\right),& \hspace{5mm} &x^- \not\in S \\
x^+ &= \Delta(x^-,\Lambda),& &x^- \in S
\end{alignedat}
    
    \end{cases}
\vspace{-2mm} 
\end{equation}
}
where $x^-$ and $x^+$ are pre- and post-impact state, 
$\Lambda$ is the impulse of swing foot collision,
\edit{$f$ is the continuous-time dynamics, }
$\Delta$ is the discrete mapping at the touchdown event, 
and $S$ is the surface in the state space where the event must occur \cite{hurmuzlu1994rigid, grizzle2014models}.

Cassie's legs contain four four-bar linkages -- two around the shin links and the other two around the tarsus links.
We simplify the model by lumping the mass of the rods of the tarsus four-bar linkages into the toe bodies, while the shin linkages are modeled with fixed-distance constraints. 
To simplify the model further, we assume Cassie's springs are infinitely stiff (or equivalently no springs), in which case $q \in \mathbb{R}^{19}$ and $v \in \mathbb{R}^{18}$. 
This assumption has been successfully deployed by other researchers \cite{hereid2018rapid}, and it is necessary for the coarse integration steps in the trajectory optimization\footnote{Reher et al. \cite{reher2019dynamic} showed 7 times increase in solve time when using the full Cassie model (with spring dynamics) in trajectory optimization.
Additionally, Cassie's spring properties can change over time and are hard to identify accurately, which discourages the use of the spring model on Cassie.} in Section \ref{sec:trajopt_background}.
\edit{
We also use this assumption in Section \ref{sec:closed_loop_eval} when comparing the ROM performances between the trajectory optimization and simulation.}

\vspace{-1mm}

\subsection{Trajectory Optimization} \label{sec:trajopt_background}
This paper will heavily leverage trajectory optimization within the inner loop of a bilevel optimization problem.
We briefly review it here, but the reader is encouraged to see \cite{Betts01} for a more complete description.
Generally speaking, trajectory optimization is a process of finding state $x(t)$ and input $u(t)$ that minimize some measure of cost $h$ while satisfying a set of constraints $C$.
Following the approach taken in prior work \cite{Posa13, Posa2016}, we explicitly optimize over state, input, and constraint (contact) forces $\lambda(t)$,
\vspace{-2mm}
\begin{equation}\label{eq:trajopt}
    \begin{array}{cl}
     \underset{x(t),u(t),\lambda(t)}{\text{min}} & \displaystyle\int_{t_0}^{t_f} h(x(t),u(t)) dt  \\
     \text{s.t.} & \dot{x}(t) = f(x(t),u(t),\lambda(t)), \\
      & C (x(t),u(t),\lambda(t)) \leq 0,  \\
    \end{array}
\end{equation}
where $f$ is the dynamics of the system, $\lambda$ are the forces required to satisfy holonomic constraints \edit{(inside $C\leq 0$)}, and $t_0$ and $t_f$ are the initial and the final time respectively.
Standard approaches  discretize in time, formulating \eqref{eq:trajopt} as a finite-dimensional nonlinear programming problem.
For the purposes of this paper, any such method would be appropriate, while we use DIRCON \cite{Posa2016} to address the closed kinematic chains of the Cassie robot.
DIRCON transcribes the infinite dimensional problem in Eq. (\ref{eq:trajopt}) into a finite dimensional nonlinear problem 
\vspace{-2mm}
\begin{equation}\label{eq:disc_trajopt}
    \begin{array}{cl}
     \underset{w}{\text{min}} & \displaystyle\sum_{i=1}^{n-1} \frac{1}{2} \big( h(x_i,u_i) + h(x_{i+1},u_{i+1}) \big) \delta_i  \\
     \text{s.t.} & f_c(x_i,x_{i+1},u_i,u_{i+1},\lambda_i,\lambda_{i+1},\delta_i,\alpha_i)=0, \\
      & \hspace{45mm} i=1,...,n-1\\
      & C (x_i,u_i,\lambda_i) \leq 0, \hspace{19mm} i=1,...,n\\
    \end{array}
\vspace{-2mm}
\end{equation}
where $n$ is the number of knot points, 
$f_c$ is the collocation constraint for dynamics, 
\edit{
$C \leq 0$ contains all the other constraints such as the four-bar-linkage kinematic constraints, }
\edit{$\delta_i$ is a constant time interval between knot point $i$ and $i+1$, }
and the decision variables are
\edit{
\begin{equation*}
\begin{alignedat}{2}
  \hspace{5mm} w = [x_1,...,x_n,&u_1,...,u_n, \lambda_1,...,\lambda_n, \alpha_1, ..., \alpha_{n-1}] \in \mathbb{R}^{n_w},
\end{alignedat}
\end{equation*}
}
where $\alpha_1, ..., \alpha_{n-1}$ are slack variables specific to DIRCON.
Eq. \eqref{eq:disc_trajopt} uses the trapezoidal rule to approximate the integration of the running cost in Eq. \eqref{eq:trajopt}, simplifying the selection of decision variables at knot points for evaluating the function $h$.

A large-scale nonlinear optimization problem such as \eqref{eq:disc_trajopt} can be difficult to solve. 
To improve the convergence of the optimization, we manually scale the decision variables, constraints and the cost function, and we also add regularization terms (see Appendix \ref{sec:heuristics_in_trajopt} for details).

\subsection{Bilevel Optimization}\label{sec:bilevel_optimization_intro}

Since our  formulation can be broadly categorized as bilevel optimization  \cite{bracken1973mathematical}, we here briefly review its basics.  

The basic structure of our bilevel optimization problem is written as 
\begin{equation}
\begin{aligned}
    \min_{\theta} & \left[ \sum_{i} \min_{w} \Psi_i(w, \theta) \right]
\end{aligned}
\end{equation}
The goal is to minimize the outer-level objective function $\sum_{i} \Psi_i(w_i^*(\theta), \theta)$ with respect to $\theta$, where $w_i^*(\theta)$ is obtained by minimizing the inner-level objective $\Psi_i(w, \theta)$ parameterized by $\theta$. 
Bilevel optimization has recently been used  in  various applications such as meta-learning \cite{franceschi2018bilevel}, reinforcement learning \cite{rajeswaran2020game},  robotics \cite{jin2022learning,pfrommer2021contactnets}, etc. 

Solving a bilevel program is generally NP-hard \cite{sinha2017review}. 
There are two types of methods  to approach  bilevel optimization.
The first type is  constraint-based \cite{hatz2012estimating, shi2005extended}, where the key idea is to replace the inner level optimization with its optimality condition (such as the KKT conditions \cite{kuhn1951}), and finally solve a ``single-level'' constrained optimization. 
However, those  methods are difficult to apply to the problem of this paper, because our inner-level is a trajectory optimization, and  replacing it with its optimality condition will introduce a large number of dual variables and co-states, dramatically increasing the size of the single-level optimization.
The second type is gradient-based \cite{jin2020pontryagin,domke2012generic}. 
The idea  is  to maintain and solve the inner-level optimization, and then update the outer-level decision variable by  differentiating through the inner-level solution  using graph-unrolling approximation \cite{domke2012generic,das2021model} or implicit function theorem \cite{krantz2002implicit}.  
Compared to  constraint-based methods, gradient-based methods maintain the bilevel structure and make   bilevel optimization more  tractable and efficient to solve.

In this paper, we use the Envelope Theorem \cite{afriat1971theory, takayama1985mathematical} and exploit the fact that our problem uses the same objective functions in the outer level and the inner level.
This structure enables us to develop a more efficient gradient-based method (the second type) to solve our problem. 
Specifically, the gradient of the outer-level objective does not require differentiating the solution of the inner-level optimization with respect to the parameters.  
This leads to two numerical advantages of our method over existing  gradient-based methods. 
First, our method bypasses the computationally intensive implicit theorem, which requires the  inverse of Hessian of the inner-level  optimization. 
Second, our method leverages the inner-loop solver's understanding of active and inactive constraints, avoiding implementing the algorithm ourselves and avoiding tuning parameters such as the active set tolerance.

\section{Model Optimization}\label{sec:approaches}

\begin{figure}[t]
 \centering
 \includegraphics[width=1\linewidth]{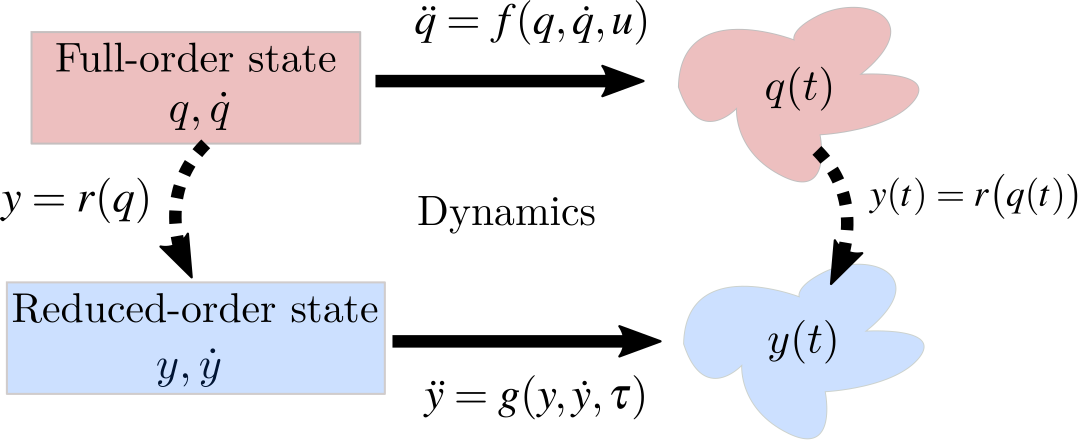}  \vspace{-5mm}
 \caption{  
 Relationship of the full-order and reduced-order models. The generalized positions $q$ and $y$ satisfy the embedding function $r$ for all time, and the evolution of the velocities $\dot{q}$ and $\dot{y}$ respects the dynamics $f$ and $g$, respectively.}
 \label{fig:mapping}
\end{figure}

In this section, we propose a definition of reduced-order models, along with a notion of quality (or cost) for such models.
We then introduce a bilevel optimization algorithm to optimize within our class of models.

\subsection{Definition of Reduced-order Models}
\label{sec:definition}
Let $q$ and $u$ be the generalized position and input of the full-order model, and let $y$ and $\tau$ be the generalized position and input of the reduced-order model.
We define a reduced-order model $\mu$ of dimension $n_y$ by two functions -- an embedding function $r: q \mapsto y$ and the second-order dynamics of the reduced-order model $g(y, \dot y, \tau)$.
That is, 
\vspace{-0mm}
\begin{equation}\label{eq:model_def}
\begin{aligned}
\mu & \triangleq (r,g),
\end{aligned}
\vspace{-2mm}
\end{equation}
with 
\vspace{-2mm}
\begin{subequations}\label{eq:model}
\begin{align}
y&=r(q),  \label{eq:model_kin}\\
\ddot{y}&= g(y,\dot{y},\tau), \label{eq:model_dyn}
\vspace{-2mm}
\end{align}
\end{subequations}
where $\dim y < \dim q$ and $\dim \tau \leq \dim u$. 
As an example, to represent SLIP, $r$ is \edit{the center of mass position relative to the foot}, $g$ is the spring-mass dynamics, and $\dim \tau = 0$ as SLIP is passive.
\edit{
Additionally, we note that the choices of $r$ and $g$ are independent of each other.
For example, LIP and SLIP share the same $r$, but they have different dynamics function $g$ (one has zero vertical acceleration and the other is the spring-mass dynamics). }

The embedding function $r$ can explicitly include the left or right leg of the robot (e.g. choosing left leg as support leg instead of right leg), in which case there will be two reduced-order models.
In this paper, we assume to parameterize over left-right symmetric reduced-order models. 
As such, we explicitly optimize over a model corresponding to left-support, which will then be mirrored to cover both left and right-support phases. 
The details of this mirroring operation can be found in Appendix \ref{sec:mirror_rom_appendix}.

Fig. \ref{fig:mapping} shows the relationship between the full-order and the reduced-order models. 
If we integrate the two models forward in time with their own dynamics, the resultant trajectories will still satisfy the embedding function $r$ at any time in the future.

\subsection{Problem Statement}
As shown in the left half of Fig. \ref{fig:outline}, the goal is to find an optimal model $\mu^*$, given a distribution $\Gamma$ over a set of tasks. 
The distribution could be provided \textit{a priori} or estimated via the output of a higher-level motion planner.
The tasks might include anything physically achievable by the robot, such as walking up a ramp at different speeds, turning at various rates, jumping, running with a specified amount of energy, etc.
The goal, then, is to find a reduced-order model that enables low-cost motion over the space of tasks,
\begin{equation}\label{eq:model_opt}
    \mu^* = \underset{\mu \in \boldsymbol{M}}{\text{argmin}} \  \mathbb{E}_{\gamma} \left[ \mathcal{J}_{\gamma}(\mu) \right], 
\end{equation}
where $\boldsymbol{M}$ is the model space, $\mathbb{E}_\gamma$ takes the expected value over $\Gamma$, and $\mathcal{J}_{\gamma}(\mu)$ is the cost required to achieve the tasks $\gamma \sim \Gamma$ while the robot is restricted to a particular model $\mu$. 

With our model definition in Eq. (\ref{eq:model_def}), the problem in Eq. \eqref{eq:model_opt} is infinite dimensional over the space of embedding and dynamics functions, $r$ and $g$.
To simplify, we parametrize  $r$ and  $g$ with basis functions $\{\phi_{e,i} \mid i = 1, \ldots, n_e\}$ and $\{\phi_{d,i} \mid i = 1, \ldots, n_d\}$ with linear weights $\theta_e \in \mathbb{R}^{n_y \cdot n_e}$ and $\theta_d \in \mathbb{R}^{n_y \cdot n_d}$.
Further assuming that the dynamics are affine in $\tau$ with constant multiplier,  $r$ and $g$ are given as
\begin{subequations}\label{eq:parametrized_rom}
\begin{align}
  y&=r(q;\theta_e) \hspace{7.2mm} = \Theta_e \phi_e (q), \label{eq:parametrized_rom_1}\\
  \ddot{y}&= g(y,\dot{y},\tau; \theta_d)=  \Theta_{d} \phi_{d} (y,\dot{y}) + B_y\tau, \label{eq:parametrized_rom_2}
\end{align}
\end{subequations}
where $\Theta_e \in \mathbb{R}^{n_y \times n_e}$ and $\Theta_d \in \mathbb{R}^{n_y \times n_d}$ are $\theta_e$ and $\theta_d$ arranged as matrices, 
$\phi_e = [\phi_{e,1},\ldots,\phi_{e,n_e}]$, $\phi_d = [\phi_{d,1},\ldots,\phi_{d,n_d}]$, 
and $B_y \in \mathbb{R}^{n_y \times n_\tau}$.
\edit{
For simplicity, we choose a constant value for $B_y$. 
Observing that physics-based rigid-body models lead to state-dependent values for $B_y$, one can also extend this method by parameterizing $B_y(y, \dot{y})$.
}
Moreover, while we choose linear parameterization here, any differentiable function approximator (e.g. a neural network) can be equivalently used.

Let the model parameters be $\theta = [\theta_e, \theta_d] \in \mathbb{R}^{n_t}$. 
Eq. (\ref{eq:model_opt}) can be rewritten as  
\begin{equation}\label{eq:parametrized_model_opt}
    \theta^* = \underset{\theta}{\text{argmin}} \  \mathbb{E}_{\gamma} \left[ \mathcal{J}_{\gamma}(\theta) \right]. 
    \tag{O}
\end{equation}
From now on, we work explicitly in $\theta$, rather than $\mu$.
\edit{
As we will see in the next section, $\mathcal{J}_{\gamma}(\theta)$ is an optimal cost of a trajectory optimization problem, making Eq. \eqref{eq:parametrized_model_opt} a bilevel optimization problem.
Additionally, given the parameterization in Eq. \eqref{eq:parametrized_rom}, the ROM dimension $n_y$ is fixed during the model optimization.
}

\subsection{Task Evaluation}\label{sec:task_eval}
We use trajectory optimization to evaluate the task cost $\mathcal{J}_{\gamma}(\theta)$.
Under this setting, the tasks $\gamma$ are defined by a cost function $h_\gamma$ and task-specific constraints $C_\gamma$. 
$\mathcal{J}_{\gamma}(\theta)$ is the optimal cost to achieve the tasks while simultaneously respecting the embedding and dynamics given by $\theta$.
We note that the cost function $h_\gamma$ is a function of the full model, although we occasionally refer to the cost evaluated by this function as the ROM performance because the ROM is embedded in the full model.

The resulting optimization problem is similar to (\ref{eq:disc_trajopt}), but contains additional constraints and decision variables for the reduced-order model embedding,
\begin{equation}\label{eq:model_trajopt_all_symbols}
\begin{array}{rl}
    \hspace{-2mm} \mathcal{J}_{\gamma}(\theta) \triangleq \ \underset{w}{\text{min}} \hspace{-2mm} & \displaystyle\sum_{i=1}^{n-1} \frac{1}{2} \big( h_\gamma(x_i,u_i) + h_\gamma(x_{i+1},u_{i+1}) \big) \delta_i  \\
      \text{s.t.} \hspace{-1.5mm} & f_c(x_i,x_{i+1},u_i,u_{i+1},\lambda_i,\lambda_{i+1},\delta_i,\alpha_i)=0, \\
      & \hspace{38.7mm} i=1,\ldots,n-1 \\
      & g_c\left(x_i,u_i,\lambda_i,\tau_i;\theta\right) = 0, \hspace{4.2mm} i=1,\ldots,n \\
      & C_\gamma (x_i,u_i,\lambda_i) \leq 0, \hspace{11.2mm} i=1,\ldots,n \\
\end{array}
\end{equation}
where $f_c$ and $g_c$ are dynamics constraints for the full-order and reduced-order dynamics, respectively. 
The decision variables are $w~=~[x_1,...,x_n,$ $u_1,...,u_n,$ $ \lambda_1,...,\lambda_n,$ $ \tau_1,...,\tau_n,$  $\alpha_1, ..., \alpha_{n-1}],$
noting the addition of $\tau_i$.

The formulation of dynamics and holonomic constraints of the full-order model are described in \cite{Posa2016}, while the reduced-order constraint $g_c$ is 

\begin{equation}\label{eq:gc_constraint}
\begin{alignedat}{2}
& g_c = \ddot{y}_i - g(y_i,\dot{y}_i,\tau_i; \theta_d) = 0 \\
\Rightarrow \ & g_c =  J_i\dot{v}_i + \dot{J}_i v_i - g(y_i,\dot{y}_i,\tau_i; \theta_d) = 0 \\
\end{alignedat}
\end{equation}
where
\edit{
\begin{equation*}
\begin{alignedat}{2}
y_i & =   r (q_i;\theta_e), 
\ \ \ \dot{y}_i =  \frac{\partial r (q_i;\theta_e)}{\partial q_i} \dot{q}_i, 
\ \ \ J_i =  \frac{\partial r (q_i;\theta_e)}{\partial q_i} , \text{\ \ \ and}\\
\dot{v}_i & = M(q_i)^{-1} \left( f_{cg} (q_i,v_i) + Bu_i + J_h(q_i)^T \lambda_i + \tau_{app}(q_i,v_i) \right).
\end{alignedat}
\end{equation*}
}
The constraint $g_c = 0$ not only explicitly describes the dynamics of the reduced-order model but also implicitly imposes the embedding constraint $r$ via the variables $y$ and $\dot{y}$.
Therefore, the problem \eqref{eq:model_trajopt_all_symbols} is equivalent to simultaneous optimization of full-order and reduced-order trajectories that must also be consistent with the embedding $r$.

\edit{
For readability, we rewrite Eq. \eqref{eq:model_trajopt_all_symbols} as 
\begin{equation}\label{eq:model_trajopt}
\begin{array}{rl}
    \hspace{-2mm} \mathcal{J}_\gamma(\theta) = \ \underset{w}{\text{min}} \hspace{-2mm} & \tilde{h}_\gamma(w)  \\
      \text{s.t.} \hspace{-1.5mm} & \tilde{f}_\gamma(w, \theta) \leq 0, \\
\end{array}
\tag{TO}
\end{equation}
where $\tilde{h}_\gamma$ is the cost function of Eq. \eqref{eq:model_trajopt_all_symbols} and $\tilde{f}_{\gamma} \leq 0$ encapsulates all the constraints in Eq.  \eqref{eq:model_trajopt_all_symbols}.
In Section \ref{sec:closed_loop_eval}, we will use $\tilde{h}_\gamma$ to evaluate both the open-loop and closed-loop performance.
}

\edit{
\begin{remark}
Model optimization can change the physical meaning of a ROM. Regardless, if $J_iM^{-1}B$ (which maps a full model input $u$ to a reduced-order acceleration $\ddot{y}_i$) has full row rank, the ROM can be exactly-embedded into the full model.
\end{remark}
}

\subsection{Bilevel Optimization Algorithm}

Since there might be a large or infinite number of tasks $\gamma \sim \Gamma$ in Eq. \eqref{eq:parametrized_model_opt}, solving for the exact solution is often intractable. 
Therefore, we use stochastic gradient descent to solve Eq. \eqref{eq:parametrized_model_opt} (specifically in the outer optimization, as opposed to the inner trajectory optimization).
That is, we sample a set of tasks from the distribution $\Gamma$ and optimize the averaged sample cost over the model parameters $\theta$. 

The full approach to (\ref{eq:parametrized_model_opt}) is outlined in Algorithm \ref{alg:model_opt}. 
Starting from an initial parameter seed $\theta_0$, $N$ tasks are sampled, and the cost for each task $\mathcal{J}_{\gamma_j}(\theta)$ is evaluated by solving the corresponding trajectory optimization problem (\ref{eq:model_trajopt}). 

To compute the gradient  $\nabla_\theta \left[ \mathcal{J}_{\gamma_j}(\theta) \right]$, we previously \cite{chen2020optimal} adopted an approach based in sequential quadratic programming. 
It introduced extra parameters (e.g. tolerance for determining active constraints) and required solving a potentially large and ill-conditioned system of linear equations which can take minutes to solve to good accuracy.
Here, we take a new approach where we apply the Envelope Theorem and directly derive the analytical gradient $\nabla_\theta \left[ \mathcal{J}_{\gamma_j}(\theta) \right]$ shown in Corollary \ref{thm:env} (also see Section \ref{sec:bilevel_optimization_intro}).

\edit{

\begin{proposition}[Differentiability Condition \cite{jin2021safe}] \label{prop:diffCondition}
Assume $\tilde{h}$ and $\tilde{f}$ are continuously differentiable functions, 
and consider an optimization problem 
\begin{equation}\label{eq:model_trajopt_with_simple_notation}
\begin{array}{rl}
    \hspace{-2mm} \tilde{\mathcal{J}}(\theta) = \ \underset{w}{\text{min}} \hspace{-2mm} & \tilde{h}(w, \theta)  \\
      \text{s.t.} \hspace{-1.5mm} & \tilde{f}(w, \theta) \leq 0, \\
\end{array}
\end{equation}
where $\tilde{\mathcal{J}}(\theta)$ is the optimal cost of the problem.
Let $w^*(\theta)$ be the optimal solution to Eq. \eqref{eq:model_trajopt_with_simple_notation}.
$w^*$ is differentiable with respect to $\theta$ if the following conditions hold:
\begin{enumerate}
\item the second-order optimality condition for Eq. \eqref{eq:model_trajopt_with_simple_notation}, 
\item linear independence constraint qualification (LICQ), and 
\item strict complementarity at $w^*$.
\end{enumerate}
\end{proposition}

\begin{theorem}[Envelope Theorem \cite{riley2012essential}] \label{thm:env_original}
Assume the problem in Eq. \eqref{eq:model_trajopt_with_simple_notation} satisfies the differentiability condition.
The gradient of the optimal cost $\tilde{\mathcal{J}}(\theta)$ with respect to $\theta$ is 
\begin{equation} \label{eq:gradient_by_env_thm_original}
\nabla_\theta \left[ \tilde{\mathcal{J}}(\theta) \right] = 
\frac{\partial \tilde{h}(w^*, \theta)}{\partial \theta} + 
{\lambda^{*}}^T \frac{\partial \tilde{f}(w^*, \theta)}{\partial \theta},
\end{equation}
where $\lambda^*$ is the dual solution to Eq. \eqref{eq:model_trajopt_with_simple_notation}.
\end{theorem}

\begin{corollary} \label{thm:env}
The gradient of the optimal cost of \eqref{eq:model_trajopt} is 
\begin{equation}\label{eq:gradient_by_env_thm}
\nabla_\theta \left[ \mathcal{J}_{\gamma}(\theta) \right] = 
{\lambda^{*}}^T \frac{\partial \tilde{f}_{\gamma}(w^*, \theta)}{\partial \theta},
\end{equation}
where $w^*$ and $\lambda^*$ are respectively the primal and the dual solution to \eqref{eq:model_trajopt}.
\end{corollary}
\begin{proof}
The proof follows directly from Theorem \ref{thm:env_original}. Note that the cost function in \eqref{eq:model_trajopt} is independent of $\theta$, in which case the first term of Eq. \eqref{eq:gradient_by_env_thm_original} becomes 0.
\end{proof}

}

\begin{algorithm}[t!]
\caption{Reduced-order model optimization}
\begin{algorithmic}[1]\label{alg:model_opt}
\renewcommand{\algorithmicrequire}{\textbf{Input:}}
\renewcommand{\algorithmicensure}{\textbf{Output:}}
\REQUIRE  Task distribution $\Gamma$ and step size $\alpha$ 
\ENSURE  $\theta^*$
\\ \textit{Model initialization}
\STATE $\theta \leftarrow \theta_0$
\\ \textit{Model optimization}
\REPEAT
\STATE Sample $N$ tasks from $\Gamma$ $\Rightarrow \gamma_j, \ j = 1, ..., N$ 
\FOR {$j=1,\ldots,N$}
\STATE Solve (\ref{eq:model_trajopt}) to get $\mathcal{J}_{\gamma_j}(\theta)$
\STATE Compute $\nabla_\theta \left[\mathcal{J}_{\gamma_j}(\theta) \right]$ by Eq. \eqref{eq:gradient_by_env_thm}
\ENDFOR
\STATE Average the gradients $\Delta \theta = \frac{\sum_{j=1}^{N} \nabla_\theta \left[ \mathcal{J}_{\gamma_j}(\theta) \right]}{N}$
\STATE \textit{Gradient descent} $\theta \leftarrow \theta - \alpha\cdot \Delta \theta$
\UNTIL convergence
\RETURN $\theta$
\end{algorithmic}
\end{algorithm}

\edit{
We note that there is, in general, no guarantee on global convergence when using Eq. \eqref{eq:gradient_by_env_thm_original} in a gradient descent algorithm, 
except for simple cases where $\tilde{h}$ and $\tilde{f}$ are convex functions in $(w, \theta)$ \cite{Boyd04a}.
As for local convergence towards a stationary point, 
the gradient descent with Eq. \eqref{eq:gradient_by_env_thm_original} is guaranteed to converge with a sufficiently small step size.
While there is no guarantee that the differentiability condition in Proposition \ref{prop:diffCondition} holds everywhere (in fact we expect it to fail under certain conditions), in practice we have observed that Algorithm 1 reliably converges. Additionally, the accuracy of the gradient $\nabla_\theta \left[ \mathcal{J}_{\gamma}(\theta) \right]$ is bounded by the accuracy of the primal and dual solutions to  \eqref{eq:model_trajopt} \cite{jin2021safe}.
In practice, we observed that the gradient was accurate enough (showing local convergence behavior) with the default optimality tolerance and constraint tolerance given by solvers like SNOPT \cite{gill2005snopt}.
}

\edit{
In Algorithm \ref{alg:model_opt}, the sampled tasks can sometimes be infeasible for the trajectory optimization problem due to a poor choice in ROM or numerical difficulties when solving \eqref{eq:model_trajopt}. 
In these cases, we do not include these samples in the gradient update step. 
This is a reasonable approach as we expect that optimizing the ROM for nearby tasks simultaneously improves performance for the failed task by continuity. 
This does have the potential to break the optimization process if large regions of the task space were infeasible, but in practice we have found this sample-rejection procedure robust enough to the occasional numerical difficulties. 
}

The model optimization in Algorithm \ref{alg:model_opt} is deemed to have converged if the norm of the average gradient of the sampled costs falls below a specified threshold.
\edit{
This threshold can be set on a case by case basis, depending on the robot models, tasks, etc. 
In our experiments, we simply look at the cost-iteration plots (e.g. Fig \ref{fig:cassie_3dlipm_cost_comparison}) and terminate the optimization when the cost has stopped decreasing visibly.
}

\newcommand*{\exDefault}{1}
\newcommand*{\exFourthOrder}{2}
\newcommand*{\exAccelCost}{3}
\newcommand*{\exAccelCostHardTask}{4}
\newcommand*{\exBiggerTaskSpaceNL}{5}

\begin{table*}[t!]
\centering
\edit{
\begin{tabular}{ || c || C{1.5cm} | C{1.5cm} | C{1.5cm} | C{1.5cm} | C{1.5cm} || } 
 \hline
 Example \# & \exDefault & \exFourthOrder & \exAccelCost & \exAccelCostHardTask &  \exBiggerTaskSpaceNL \\
   \hline
 \hline
 Stride length (m) & \multicolumn{3}{c|}{[-0.4, 0.4] } & [0.3, 0.4] & \multicolumn{1}{c||}{[0.0, 0.42]} \\ 
 \hline
 Pelvis height (m)  & \multicolumn{4}{c|}{[0.87, 1.03]} & \multicolumn{1}{c||}{0.8} \\ 
 \hline
 Ground incline (rad) & \multicolumn{4}{c|}{0} & \multicolumn{1}{c||}{[-0.35, 0.35]} \\ 
 \hline
 Turning rate (rad/s) & \multicolumn{4}{c|}{0} & \multicolumn{1}{c||}{[-0.72, 0.72]} \\ 
 \hline
 Stride duration (s) & \multicolumn{4}{c|}{0.35} & \multicolumn{1}{c||}{0.35} \\ 
 \hline
 Parameterize $(r,g)$? & \multicolumn{4}{c|}{both $(r,g)$} & \multicolumn{1}{c||}{only $g$} \\ 
 \hline
 Monomial order $n_\phi$ & 2 & 4 & 2 & 2 & 4 \\ 
 \hline
  Dominant cost in $J_\gamma$ & $u$ & $u$ & $\dot{v}$ & $\dot{v}$ & $u$ \\ 
 \hline
 Cost reduction & 22.8\% & 20.7\% & 27.6\% & 38.2\% & 22.4\% \\ 
 \hline
\end{tabular}
\vspace{1mm}
\caption{Examples of model optimization. 
This table includes the task space used to train models (uniform task distribution), the highest order of the monomials of basis functions, the dominant term of the cost function $J_\gamma$, and the cost reduction percentage (relative to the cost of the initial model).}
}
  \vspace{-3mm}
\label{table:task_space}
\end{table*}

\subsection{Examples of Model Optimization}\label{sec:model_optimization_examples}

In the trajectory optimization problem in Eq. \eqref{eq:model_trajopt}, we assume the robot walks with instantaneous change of support. 
That is, the robot transitions from right support to left support instantaneously, and vice versa. 
We consider only half-gait periodic motion, and so include right-left leg alternation in the impact map $\Delta$.

We solve the problems \eqref{eq:model_trajopt} in parallel in each iteration of Algorithm \ref{alg:model_opt} using the SNOPT toolbox \cite{gill2005snopt}.
All examples were generated using the Drake software toolbox \cite{Drake2016} and source code is available in the link provided in the Abstract.

\begin{figure}[t!]
 \centering
   \includegraphics[width=0.45\linewidth]{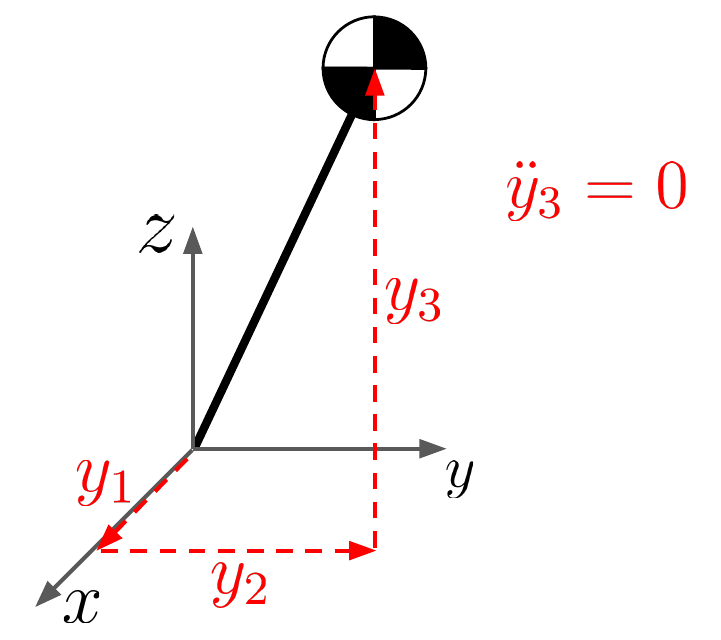}
 \vspace{-1mm}
 \caption{The linear inverted pendulum (LIP) model. 
 It is a point mass model of which height is restricted in a plane. 
 The point mass and the origin of this model correspond to the center of mass and the stance foot of the robot, respectively. In the examples of this paper, we initialize the reduced-order model to the LIP model during model optimization.}
 \label{fig:initialrom}
\end{figure}

\subsubsection{Initialization and parameterization of ROM}
To demonstrate Algorithm \ref{alg:model_opt}, we optimize for 3D reduced-order models on Cassie. 
The models are initialized with a three-dimensional LIP, of which the generalized position $y$ is shown in Fig. \ref{fig:initialrom}. 
For reference, the equations of motion of the 3D LIP model are 
\begin{equation}\label{eq:lipmwithsfdyn}
\ddot{y} = \left[\begin{array}{c}
     \ddot{y}_1\\
     \ddot{y}_2\\
     \ddot{y}_3\\
    \end{array}\right] =
    \left[\begin{array}{c}
     c_g \cdot y_1 / y_3\\
     c_g \cdot y_2 / y_3\\
     0\\
    \end{array}\right],
\end{equation}
where $c_g$ is the gravitational acceleration constant.
This model represents a point-mass body, where the body has a constant speed in the vertical direction.

We choose basis functions such that they not only explicitly include the position of the LIP, but also include a diverse set of additional terms.
That is, the basis set $\phi_e$ includes the CoM position relative to the stance foot, and monomials of $\{1, q_7, ..., q_{19}\}$ up to $n_\phi$-th order. 
Similarly, the feature set $\phi_d$ includes the terms in LIP dynamics (i.e. $c_g y_1 / y_3$ and $c_g y_2 / y_3$) and monomials of $\{1, y_1, y_2, y_3, \dot{y}_1, \dot{y}_2, \dot{y}_3\}$ up to $n_\phi$-th order.  
With these basis functions, the ROM parameters $\theta$ can be trivially initialized to match the LIP model's.

\begin{figure}[t!]
 \centering
    \includegraphics[width=0.95\linewidth]{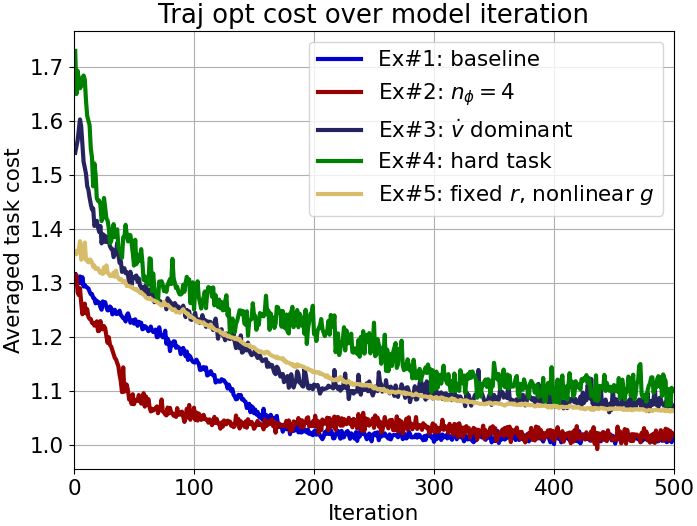}

 \caption{The averaged cost of the sampled tasks of each model optimization iteration in Examples \exDefault \ to \exBiggerTaskSpaceNL. 
 Costs are normalized by the cost associated with the full-order model (i.e. the cost of full model trajectory optimization without any reduced-order model embedding). 
 Therefore, the costs cannot go below 1. 
 The costs at iteration 1 represent the averaged costs for the robots with the embedded initial reduced-order models, LIP. 
 Note that the empirical average does not strictly decrease, as tasks are randomly sampled and are of varying difficulty.
 }
 \label{fig:cassie_3dlipm_cost_comparison}
\end{figure}

\subsubsection{Optimization Examples and Result}\label{sec:model_opt_result}

We demonstrate a few examples of model optimization and compare their results.
The examples are shown in Table \ref{table:task_space} along with their detailed settings.
\eedit{
We note that it is a common practice to fix the duration to ease the problem difficulty of \eqref{eq:model_trajopt} when dealing with high-dimensional robots \cite{Posa2016, gong2019feedback}. 
Freeing the duration can certainly be a variation of the examples and we left it for future work. 
}

The optimization results are shown in Fig. \ref{fig:cassie_3dlipm_cost_comparison}, where the costs are normalized by the optimal cost of \eqref{eq:model_trajopt} without ROM embedding (i.e. without the constraints $g_c$).
The cost function $h_\gamma$ was chosen to be the weighted sum of squares of the robot input $u$, the generalized velocity $v$ and acceleration $\dot{v}$. 
In Examples \exDefault, \exFourthOrder \ and \exBiggerTaskSpaceNL, we heavily penalize the input term which is a proxy of a robot's energy consumption.
For the other examples, we heavily penalize the acceleration $\dot{v}$.
We observed that Cassie's motions with the initial ROM are very similar among all examples.
In contrast, the motions with optimal ROMs are mostly dependent on the cost function $h_\gamma$, given the same ROM parameterization.
Compared to Example \exDefault, the optimal motion of Example \exAccelCost \ shows more vertical pelvis movement.

\edit{
A comparison between Example \exDefault \ to Example \exFourthOrder \ shows the effect of the order of the monomials $n_\phi$ in the basis function. 
We can see in Fig. \ref{fig:cassie_3dlipm_cost_comparison} that these two examples share the same starting cost, 
because the initial weights on the monomials are zeros, making the trajectory optimization problems identical. } 
Additionally we can also see that parameterizing the ROM with second-order monomials seems sufficient  for the task space of Examples \exDefault \ and \exFourthOrder, since the final normalized cost is close to 1. 

\edit{
A comparison between Example \exDefault \ to Example \exAccelCost \ shows the effect of different choices of cost function $h_\gamma$. 
The initial cost of Example \exAccelCost \ is much higher than Example \exDefault's, which we can interpret as the LIP model being more restrictive under the performance metric of Example \exAccelCost \ than that of \exDefault. }

\edit{
A comparison between Example \exAccelCost \ and Example \exAccelCostHardTask \ shows the effect of the task space. }
Example \exAccelCostHardTask's task space is a subset of Example \exAccelCost's, specifically the part of the task space with bigger stride length.
We would expect the LIP model does not perform well with big stride length, and indeed Fig. \ref{fig:cassie_3dlipm_cost_comparison} shows that the initial cost of Example \exAccelCostHardTask \ is higher than that of Example \exAccelCost.
Fortunately, a high initial cost provides us with a bigger room of potential improvement.
As we see in Table \ref{table:task_space}, Example \exAccelCost \ has higher cost reduction than Example \exDefault, and Example \exAccelCostHardTask\  has the highest cost reduction.

\edit{
In Example \exBiggerTaskSpaceNL, 
the dimension of the task space\footnote{
\edit{
The dimension here is the non-degenerate dimension, meaning the task dimension with non-zero volume. 
In Example \exBiggerTaskSpaceNL, the dimension is 3 because we vary the stride length, ground incline and turning rate.
}
} is increased compared to the other examples, 
and we only parameterize the ROM dynamics $g$. 
That is, the embedding function $r$ remains to be a simple forward kinematic function -- the center of mass position relative to the stance foot.
In this case, the algorithm was again able to find an optimal model, 
and the result shows that parameterizing only the ROM dynamics $g$ is sufficient enough for achieving near full model performance (about 5\% higher than full-order model's performance). 
}

The optimized models are capable of expressing more input-efficient motions
than the LIP model, better leveraging the natural dynamics of Cassie. The ROM optimization improves the robot performance, while maintaining the model simplicity.
We note that the optimal model, unlike its classical counterpart, does not map easily to a physical model, if the embedding function $r$ contains abstract basis functions such as monomials. 
While this limits our ability to attach physical meaning to $y$ and $\tau$, it is a sacrifice that one can make to improve performance beyond that of hand-designed approaches.

\section{MPC for a Special Class of ROM}\label{sec:mpc}

\begin{figure}[t!]
 \centering
      \includegraphics[width=1\linewidth]{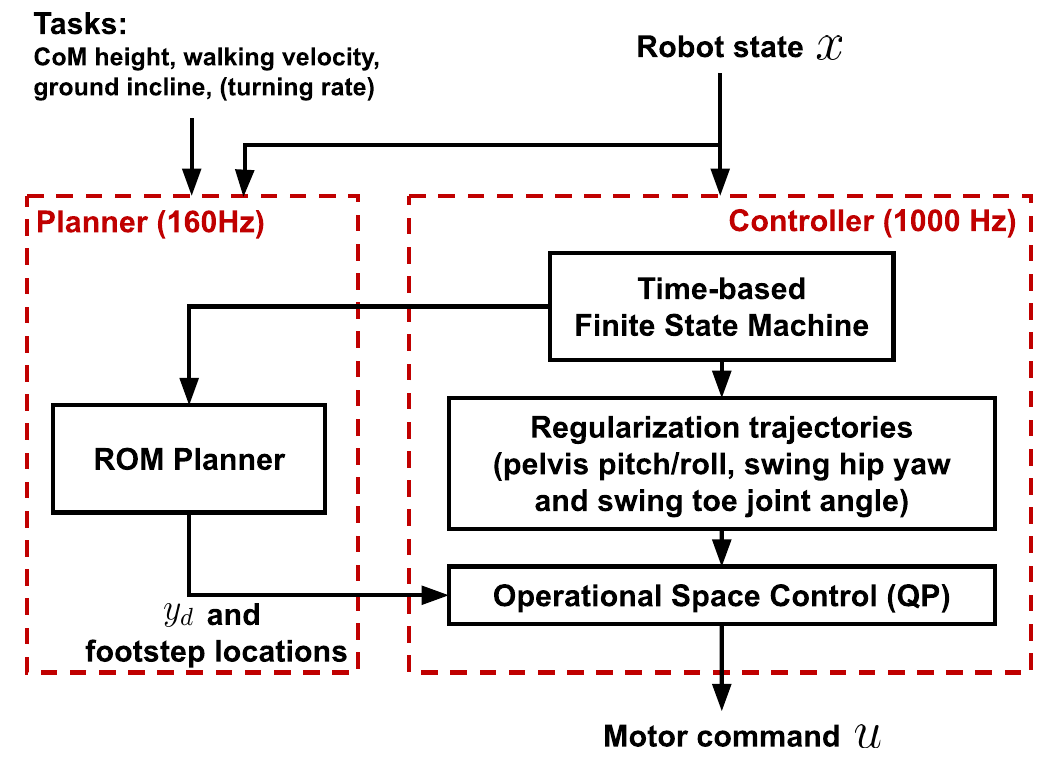}
\vspace{-5mm}

 \edit{
 \caption{
 The diagram of the model predictive control (MPC) introduced in Section \ref{sec:mpc}. 
 The MPC is composed of the controller process and the planner process, 
 and it contains a time-based finite state machine which outputs either left or right support state.
 This finite state machine determines the contact sequence of the high-level planner and the contact mode of the low-level model-based controller. 
 The high-level planner solves for the desired reduced-order model trajectories and swing foot stepping locations, given tasks (commands) and the finite state. 
 For reduced-order models without body orientation (e.g. CoM model without moment of inertia), we send the turning rate command to the controller process instead of planner process.
 Inside the controller process, the regularization trajectories are used to fill out the joint redundancy of the robot.
 These regularization trajectories are derived from simple heuristics such as maintaining a horizontal attitude of the pelvis body, having the swing foot parallel to the contact surface, and aligning the hip yaw angle with the desired heading angle.
 All desired trajectories are sent to the Operational Space Controller (OSC) which is a quadratic-programming based inverse-dynamics controller \cite{sentis2005control, wensing2013generation}.
 }
 }
 
 \label{fig:mpc_diagram}
\end{figure}

After a ROM is optimized, we embed it in the robot via an MPC to achieve desired tasks, depicted in Fig. \ref{fig:outline}.
\edit{
Specifically, in this and the next section (Sections \ref{sec:mpc} and \ref{sec:closed_loop_eval}), we build upon Example \exBiggerTaskSpaceNL.
Example \exBiggerTaskSpaceNL \ limits the ROM to a fixed embedding function $r$, the CoM position relative to the stance foot.
This physically-interpretable embedding simplifies the planner and enables a richer performance analysis in Section \ref{sec:closed_loop_eval}.
}
The planner for a general ROM will be introduced in Section \ref{sec:mpc_generic_rom}.

The MPC structure is shown in Fig. \ref{fig:mpc_diagram}. 
It contains a high-level planner in the reduced-order space (Section \ref{sec:plan_with_rom}) and a low-level tracking controller in the full-order space (Section \ref{sec:osc_implementation}). 
The high-level planner receives the robot state and tasks, and plans for the desired ROM trajectory and the footsteps of the robot.
The controller tracks these desired trajectory and footsteps, while internally using nominal trajectories to handle the system redundancy.

\begin{table*}[h!]
\centering
\begin{tabular}{ ||p{2.8cm}||C{1cm}|p{0.5cm}|p{0.5cm}|p{0.5cm}|p{0.5cm}|p{0.5cm}|p{0.5cm}|p{0.5cm}|p{0.5cm}|p{0.5cm}||  }
\hline
\multirow{2}{*}{Trajectory $y_i^{osc}$ } & \multirow{2}{*}{dim $y_i^{osc}$} & \multicolumn{3}{|c|}{cost weight $W$} & \multicolumn{3}{|c|}{$K_p$} & \multicolumn{3}{|c||}{$K_d$} \\
\cline{ 3- 11}
 & & x & y & z & x & y & z & x & y & z \\
\hhline{||= ==========||}
reduced-order model  & 3 & 0.1 & 0 & 10 & 10 & 0 & 50 & 0.2 & 0 & 1 \\
\cline{ 2- 11}
pelvis orientation  & 3 & 2 & 4 & 0.02 & 200 & 200 & 0 & 10 & 10 & 10 \\
\cline{ 2- 11}
swing foot position	& 3  & 4 & 4 & 4 & 150 & 150 & 200 & 1 & 1 & 1 \\
\cline{ 2- 11}
swing leg hip yaw joint & 1 & \multicolumn{3}{|c|}{0.5} & \multicolumn{3}{|c|}{40} & \multicolumn{3}{|c||}{0.5} \\
\cline{ 2- 11}
swing leg toe joint & 1 & \multicolumn{3}{|c|}{2} & \multicolumn{3}{|c|}{1500} & \multicolumn{3}{|c||}{10} \\
\hline
\end{tabular}
\vspace{1mm}
\caption{Trajectories and gains in the Operational Space Control (OSC)}
\label{table:tracking_gains}
\vspace{-5mm}
\end{table*}

\subsection{Planning with Reduced-order Models}\label{sec:plan_with_rom}

We formulate a reduced-order trajectory optimization problem to walk 
$n_s$ strides, using direct collocation method described in Section \ref{sec:trajopt_background} to discretize the trajectory into $n$ knot points.
Under the premise that the ROM embedding $r$ is the CoM, we further assume the ROM does not have continuous inputs $\tau$ (e.g. center of pressure) but it has discrete inputs $\tau_{fp} \in \mathbb{R}^2$ which is the stepping location of the swing foot relative to the stance foot.
Let $z = [y, \dot{y}] \in \mathbb{R}^{2n_y}$, and let $z^-$ and $z^+$ be the reduced state of pre- and post-touchdown event, respectively. 
The discrete dynamics is 
\begin{equation}\label{eq:com_discrete_dync}
z^+ = z^- + B_{fp} \tau_{fp}
\end{equation}
with 
\begin{equation*}
B_{fp} = 
\begin{bmatrix}
 -1 & 0 & 0 & 0 & 0 & 0\\
 0 & -1 & 0 & 0 & 0 & 0
\end{bmatrix}^T.
\end{equation*}
The first two rows of Eq. \eqref{eq:com_discrete_dync} correspond to the change in stance foot reference for the COM position. 
The last three rows are derived from the assumption of zero ground impact at the foot touchdown event.

To improve readability, we stack decision variables into bigger vectors $z = [y, \dot{y}] \in \mathbb{R}^{2n_y}$, 
$Z=[z_0, z_1, ..., z_{n}] \in \mathbb{R}^{2n_y (n+1)}$, 
and $\Tau_{fp}=[\tau_{fp,1}, ..., \tau_{fp,n_s}] \in \mathbb{R}^{2 n_s}$.
The cost function of the planner is quadratic and expressed in terms of $Z$ and $\Tau_{fp}$.
The planning problem is 
\begin{equation}\label{eq:rom_planning_trajopt}
    \begin{array}{cl}
     \underset{Z, \Tau_{fp}}{\text{min}} & \| Z - Z_{d} \|^2_{W_Z} + \| \Tau_{fp} \|^2_{W_\Tau}  \\
     \text{s.t.} & \text{ROM continuous dynamics (Eq. \eqref{eq:model_dyn})}, \\
      & \text{ROM discrete dynamics (Eq. \eqref{eq:com_discrete_dync}}),\\
      & C_{kinematics} (Z,\Tau_{fp}) \leq 0, \\
      & z_0 = \text{current feedback reduced-order state},  \\
    \end{array}
\end{equation}
where $W_Z$ and $W_T$ are the weights of the norms, 
$Z_{d}$ is a stack of desired states which encourage the robot to reach a goal location and regularize velocities,
and $C_{kinematics} \leq 0$ is the constraints on step lengths and stepping locations relative to the CoM.
After solving Eq. \eqref{eq:rom_planning_trajopt}, we reconstruct the desired ROM trajectory $y_d(t)$ from the optimal solution $Z^*$, and we construct desired swing foot trajectories from $\Tau_{fp}^*$ with cubic splines.

\subsection{Operational Space Controller}\label{sec:osc_implementation}

A controller commonly used in legged robots is the quadratic-programing-based operational space controller (QP-based OSC), which is also referred to as the QP-based whole body controller \cite{sentis2005control, wensing2013generation}. 
Assume there are $N_{y}$ number of outputs $y_i^{osc}(q)$, with desired outputs $y_{i,d}^{osc}(t)$, where $i=1,2,...N_y$. For each output (neglecting the subscript $i$), we can derive the commanded acceleration as the sum of the feedforward acceleration of the desired output and a PD control law   
$$\ddot{y}_{cmd}^{osc} = \ddot{y}_{d}^{osc} + K_p (y_{d}^{osc} - y^{osc}) + K_d (\dot{y}_{d}^{osc} - \dot{y}^{osc}).$$
At a high level, the OSC solves for robot inputs that minimize the output tracking errors, while respecting the full model dynamics and constraints (essentially an MPC but with zero time horizon). 
The optimal control problem of OSC is formulated as
\begin{subequations}\label{eq:osc}
\begin{align}
     \underset{\dot{v},u,\lambda, \epsilon}{\text{min}} \hspace{3mm} &  \displaystyle\sum_{i=1}^{n_y} \| \ddot{y}_i^{osc} - \ddot{y}_{i,cmd}^{osc} \|^2_{W_i} + \| u \|^2_{W_u} +  \| \epsilon \|^2_{W_\epsilon}\\
\text{s.t.} \hspace{3mm}    
      &  \ddot{y}_i^{osc} = J_i\dot{v} + \dot{J}_i v,  \ \ i = 1, ... , N_y \label{eq:osc_target} \\
      & \text{Dynamics constraint (Eq. \eqref{eq:eom})}  \\
              &  \epsilon = J_h\dot{v} + \dot{J}_hv \label{eq:osc_holonomic}\\
      &  u_{min} \leq u \leq u_{max} \\
      &  \text{Contact force constraints}  \label{eq:osc_force_constraint}
\end{align}
\end{subequations}
where $\| \cdot \|_W$ is the weighted 2-norm,  
\edit{
\eqref{eq:osc_holonomic} contains Cassie's four-bar-linkage constraints, fixed-spring constraints and relaxed contact constraints (relaxed by slack variables $\epsilon$), }
and \eqref{eq:osc_force_constraint} includes friction cone constraints, non-negative normal force constraints and force blending constraints for stance leg transition. 

Table \ref{table:tracking_gains} shows all trajectories tracked by the OSC and their corresponding gains and cost weights.
The trajectories of the reduced-order model, pelvis orientation and swing foot position are all 3 dimensional, while the hip yaw joint and toe joint of the swing foot are 1 dimensional.
The symbols (x,y,z) in Table \ref{table:tracking_gains} indicate the components of the tracking target. 
They do not necessarily mean the physical (x, y, z) axes for the reduced-order model, since the model optimization might produce a physically non-interpretable model embedding $r$.

In the existing literature of bipedal robots, robot's floating base position (sometimes the CoM position) and orientation are often chosen to be control targets.
They have 6 degrees of freedom (DoF) in total.
In the case of fully-actuated robots (i.e. robots with flat feet), there is enough control authority to servo both the position and orientation.
For underactuated robots, the existing approaches often give up tracking the trajectories in the transverse plane (x and y axis), because it is not possible to instantaneously track trajectories whose dimension is higher than the number of actuators (or we have to trade off the tracking performance). 
In this case, motion planning for discrete footstep locations is used to regulate the underactuated DoF.
In our control problem, we also face the same challenge since Cassie has line feet. 
The total dimension of the desired trajectories in Table \ref{table:tracking_gains} is 11, while Cassie only has 10 actuators.
Following the common approach, we choose not to track the second element  of the ROM in OSC, because it corresponds to the lateral position of the CoM for the initial model (a competing tracking objective to the pelvis roll angle) and maintaining a good pelvis roll tracking is crucial for stable walking.
Instead, the second element of ROM is regulated by the desired swing foot locations via the planner in Eq. \eqref{eq:rom_planning_trajopt}, even though the OSC does not explicitly track it.

\subsection{Hardware Setup and Solve Time} \label{sec:hardware_setup_and_solve_time}
We implement the MPC in Fig. \ref{fig:mpc_diagram} using the Drake toolbox \cite{Drake2016}, and the code is publicly available in the link provided in the Abstract. In hardware experiment, the MPC planner runs on a laptop equipped with Intel i7 11800H, and everything else (low-level controller, state estimator, etc) on Cassie's onboard computer. 
These two computers communicate via LCM \cite{huang2010lcm}.
A human sends walking velocity commands to Cassie with a remote controller.
Cassie is able to stably walk around with both the initial ROM and the optimal ROM (shown in the supplementary video).

The planning horizon was set to 2 foot steps with stride duration being 0.4 seconds. 
With cubic spline interpolation between knot points, we found that 4 knot points per stride was sufficient.
IPOPT \cite{wachter2006implementation} was used to solve the planning problem in Eq. \eqref{eq:rom_planning_trajopt}, 
and the solve time was on average around 6 milliseconds with warm-starts. 
We observed that this solve time was independent of the reduced-order models (initial or optimal) in our experiments.
In contrast to the ROM, similar code required tens of seconds for the simplified Cassie model for a single foot step. 
As the following sections will show, Cassie's performance (with respect to the user-specified cost function) with the optimal model is better than with the initial model.
This demonstrates that the use of ROM greatly increases planning speed, and that the optimized ROM improves the performance of the robot.

\section{Performance Evaluation and Comparison}\label{sec:closed_loop_eval}

\begin{figure*}[!t]
\edit{
\centering

\includegraphics[width=0.95\linewidth]{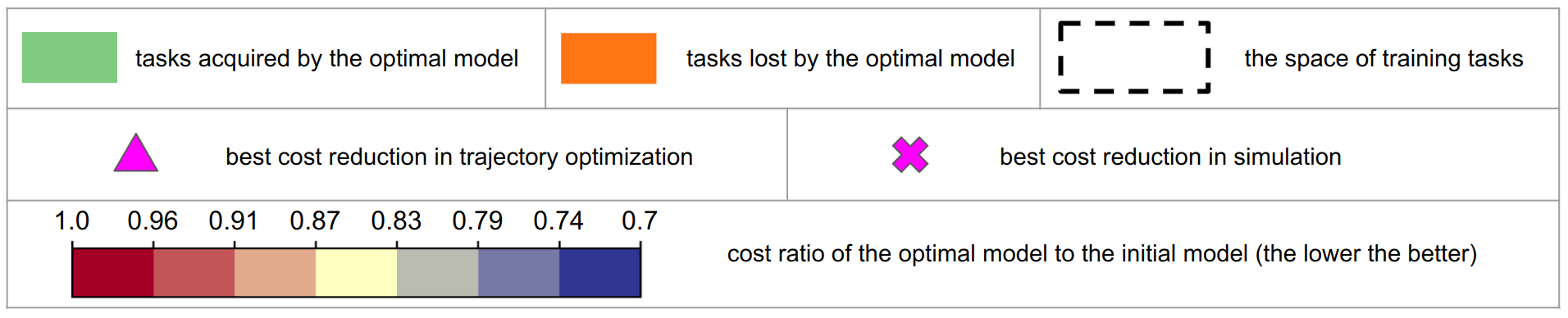}
\vspace{-3mm}

\subfloat[Trajectory optimization, \ref{env:trajopt}.]{
\includegraphics[width=0.44\linewidth]{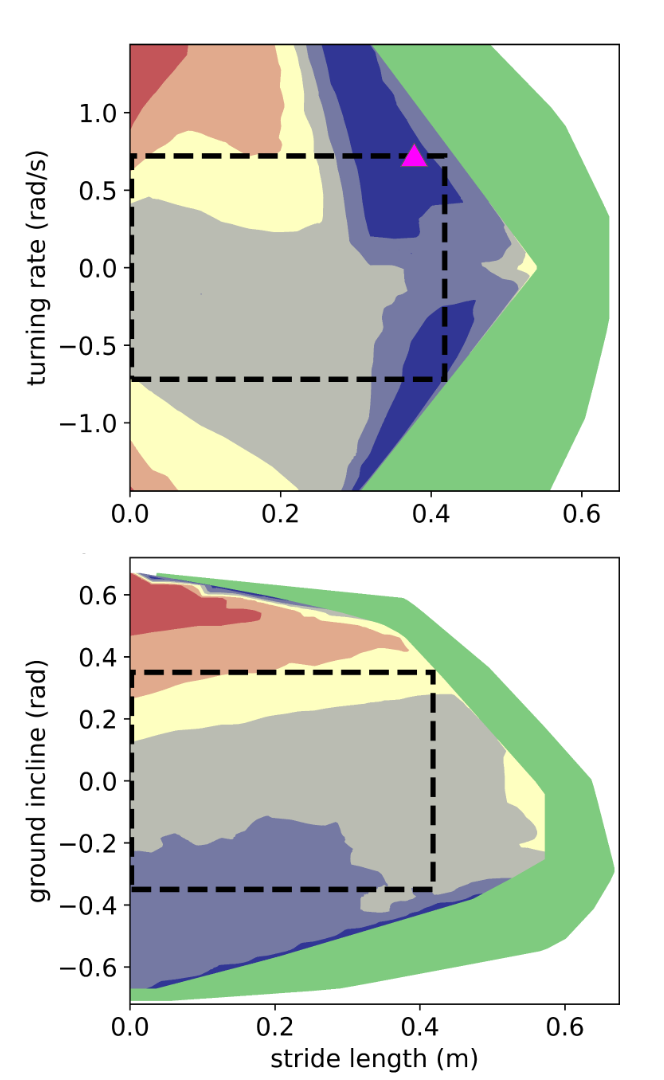}
\label{fig:openloop_cost_landscape_comparison_3d_task}}
\subfloat[Simulation, \ref{env:simplified_cassie_sim}.]{
\includegraphics[width=0.45\linewidth]{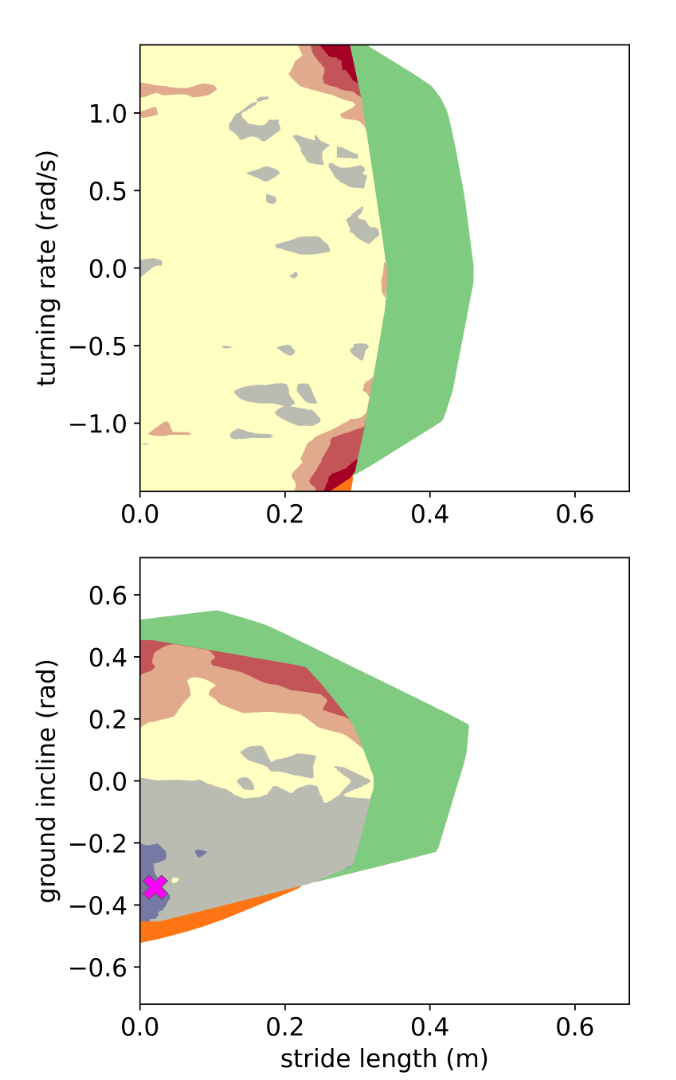}
\label{fig:closedloop_cost_landscape_comparison_3d_task}}        
		
\caption{
Cost comparison between the initial model \ref{case:initial_rom} and the optimal model \ref{case:optimal_rom}. 
Each plot shows the ratio of the optimal model's cost to the initial model's cost. 
For these examples, the ROM reduces the cost across the entire task space. 
The color scheme red-to-blue illustrates the degree to which the ROM shows improvement, with red corresponding to a minimal improvement and blue to a 30\% reduction.
The scales of the axes are the same between the trajectory optimization \ref{env:trajopt} and the simulation \ref{env:simplified_cassie_sim} for ease of comparisons.
         }
\label{fig:cost_landscape_comparison_3d_task}
}
\end{figure*}

In this section, we evaluate the performance of the robot (with respect to a user-specified cost function $\tilde{h}_\gamma$ in Section \ref{sec:task_eval}) in the following ROM settings:
\begin{enumerate}[start=1,label={(R\arabic*)}, leftmargin = 3em]
\item \label{case:no_rom} without reduced-order model embedding,
\item \label{case:initial_rom} with initial reduced-order model embedding,
\item \label{case:optimal_rom} with optimal reduced-order model embedding.
\end{enumerate}
Additionally, the evaluation is done in the following cases:
\begin{enumerate}[start=1,label={(C\arabic*)}, leftmargin = 3em]
\item \label{env:trajopt} trajectory optimization (open-loop),
\item \label{env:simplified_cassie_sim} simulation (closed-loop),
\item \label{env:real_cassie} hardware experiment with real Cassie (closed-loop), 
\end{enumerate}
where \ref{env:trajopt} is labeled as open-loop, while the others are considered closed-loop, because trajectory optimization is an optimal control method that solves for control inputs and feasible state trajectories simultaneously, without requiring a feedback controller.
Table \ref{fig:table_of_conducted_experiments} lists the experiments conducted in this section. 
We note that \ref{env:trajopt} is the same as Eq. \eqref{eq:model_trajopt} but with a different task distribution, and that \ref{case:no_rom} is only evaluated in \ref{env:trajopt} because it serves as an idealized benchmark for comparison.

\subsection{Experiment Motivations}

\subsubsection{Motivation for \ref{env:trajopt}} 
In Section \ref{sec:approaches}, we optimized for a reduced-order model given a task distribution. 
\edit{
Here, one objective is to evaluate how well the model generalizes to out-of-distribution tasks. }
Additionally, trajectory optimization provides the ideal performance benchmark for the closed-loop system to compare to.

\subsubsection{Motivation for \ref{env:simplified_cassie_sim}}
Trajectory optimization is used in Eq. \eqref{eq:model_trajopt} to find the optimal model based on the open-loop performance. 
\ref{env:simplified_cassie_sim} evaluates how well the cost reduction in open-loop can be translated to the closed-loop system with the MPC from Section \ref{sec:mpc}. 

\subsubsection{Motivation for \ref{env:real_cassie}}
\ref{env:real_cassie} evaluates how well the performance improvement can be translated to hardware.

\begin{table}[t!]
\edit{
 \centering
   \includegraphics[width=1.0\linewidth]{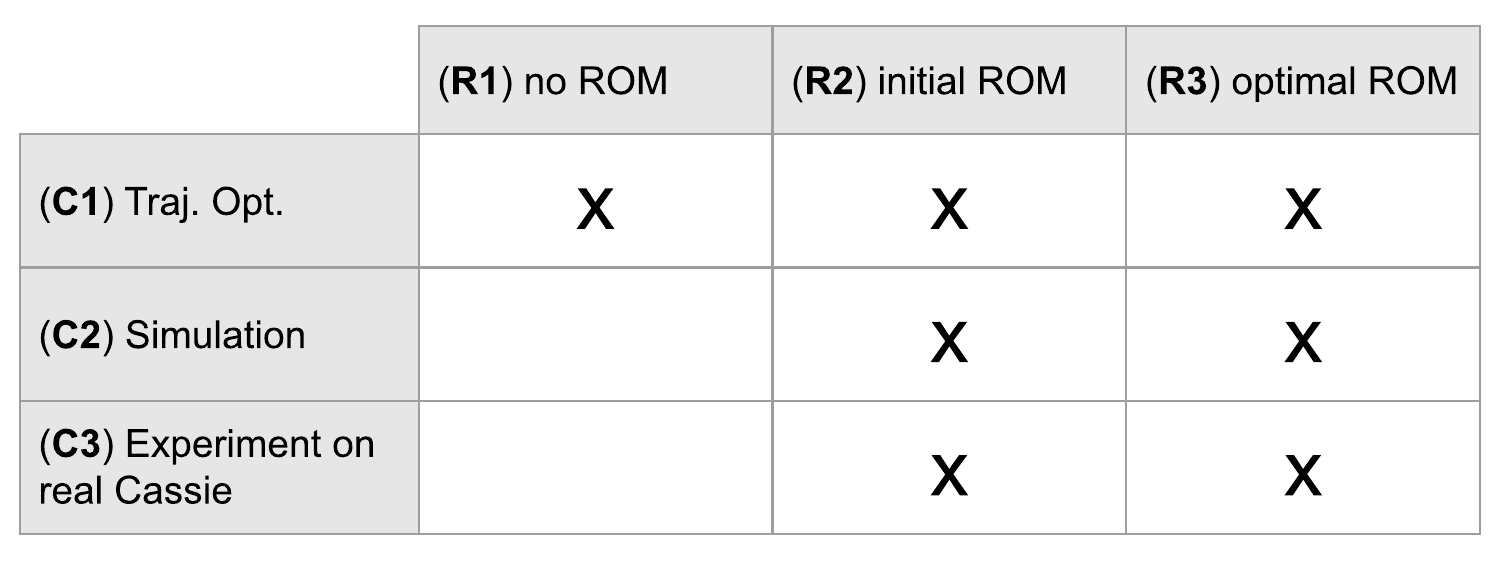}
 \vspace{-3mm}
 \caption{
 Experiments conducted in Section \ref{sec:closed_loop_eval} (marked with x)
}
 \label{fig:table_of_conducted_experiments}
 }
\end{table}

\begin{table}[t!]
\centering
\edit{
\begin{tabular}{ | c | c | } 
 \hline
 Stride length variation & $<2$ cm \\ 
 \hline
 Side stepping variation & $<3$ cm \\ 
 \hline
 Pelvis height variation & $<3$ cm \\ 
 \hline
 Pelvis yaw variation & $<0.1$ rad \\ 
 \hline
 Window size & 4 consecutive footsteps \\ 
 \hline
\end{tabular}
\vspace{1mm}
\caption{Criteria to determine periodic walking gaits}
\label{table:periodic_criteria}
}
  \vspace{-3mm}
\end{table}

\subsection{Experiment Setups}
\edit{
In all experiments in this section, we use the initial and the optimal ROM from Example \exBiggerTaskSpaceNL.
}
\edit{
Additionally, for all performance evaluations, we use the cost function $\tilde{h}_\gamma$  from Example \exBiggerTaskSpaceNL \ which mainly penalizes the joint torques. 
}

\subsubsection{\ref{env:trajopt} Trajectory Optimization}
We evaluate the open-loop performance by running the full-model trajectory optimization in Eq. \eqref{eq:model_trajopt} over a wide range of tasks (stride length, turning rate, ground incline, etc).
As a special case, \ref{env:trajopt} combined with \ref{case:no_rom} corresponds to Eq. \eqref{eq:model_trajopt} without the constraint $g_c = 0$. 

\subsubsection{\ref{env:simplified_cassie_sim} Simulation}
We use Drake simulation \cite{Drake2016}.
The MPC horizon is set to two footsteps, and the duration per step is fixed to 0.35 seconds which is the same as that of open-loop. 
Similar to \ref{env:trajopt}, we evaluate the performance at different tasks. 
\edit{
For each desired task, we run a simulation for 12 seconds and extract a periodic walking gait based on a set of criteria listed in Table \ref{table:periodic_criteria}. 
Then we compute the cost and the actual achieved tasks (stride length, turning rate, etc) of that periodic gait.
}

\edit{
\subsubsection{\ref{env:real_cassie} Hardware Experiment}
Some heuristics are introduced to the MPC in order to stabilize Cassie well. 
For example, we add a double-support phase to smoothly transition between two single-support phases by linearly blending the ground forces of the two legs. 
This is critical for Cassie, because unloading the springs of the support leg too fast when transitioning into swing phase can cause foot oscillation and bad swing foot trajectory tracking.
The double-support phase duration is set to 0.1 seconds, and the swing phase duration is decreased to 0.3 seconds, compared to the nominal 0.35 seconds of stride duration in the trajectory optimization.
}

The hardware setup is described in Section \ref{sec:hardware_setup_and_solve_time}.
During the experiment, we send commands to walk Cassie around and make sure that the safety hoist does not interfere with Cassie's motion. 
After the experiment, we apply the criteria listed in Table \ref{table:periodic_criteria} to extract periodic gait for performance evaluation.

\subsection{Turning and Sloped Walking in Simulation}\label{sec:showcase_planner}

The goal of this section is to evaluate model performance in simulation.
We use the initial and the optimal model from Example \exBiggerTaskSpaceNL \ of Section \ref{sec:model_optimization_examples}. 
The training task distribution of this model covers various turning rates, ground inclines and positive stride lengths, shown in Table \ref{table:task_space}.
During performance evaluation (both \ref{env:trajopt} and \ref{env:simplified_cassie_sim}), we increase the task space size to two times of that of training stage in order to examine the optimal model's performance on the both seen (from training) and unseen tasks.

To visualize the performance improvement, we compare the cost landscapes between the initial and the optimal model.
For \ref{env:trajopt} and \ref{env:simplified_cassie_sim}, we first derive the cost landscapes of both models using the cost function $\tilde{h}_\gamma$ and then superimpose them in terms of cost ratio (i.e. ratio of \ref{case:optimal_rom}'s cost to \ref{case:initial_rom}'s cost).
The cost landscape comparisons are shown in Fig. \ref{fig:cost_landscape_comparison_3d_task}.
The red-blue color bar represents different levels of performance improvement in terms of $\tilde{h}_\gamma$.
Green color corresponds to the tasks acquired by the optimal model (i.e. the task that the initial model cannot execute).
Orange color corresponds to the task lost by using the optimal model.
\eedit{
We see in Fig. \ref{fig:cost_landscape_comparison_3d_task} that the optimal ROM outperforms the LIP across all depicted tasks, 
as the cost ratios are smaller than 1, indicated by the color bar.
In trajectory optimization, 
the maximum cost reduction is 30\%, occurring at a stride length of approximately 0.38 meters and a turning rate of 0.71 rad/s, indicated by the pink triangle in Fig. \ref{fig:openloop_cost_landscape_comparison_3d_task}.
On the other hand, the maximum cost reduction in simulation is 23\%, observed at a stride length of approximately 0.02 meters and a -0.34 radians ground incline, indicated by the pink cross in Fig. \ref{fig:closedloop_cost_landscape_comparison_3d_task}.
This reduction in cost implies, for instance, that Cassie is able to complete the same task with 23\% less joint torque in simulation.
}

Besides the improvement in terms of cost ratio, we also observe in simulation that the optimal ROM gained new task capability (indicated by the green area).
For example, Cassie is capable of walking 54\% faster on a slope of 0.2 radian when using the optimal ROM.
Cassie also gets better in climbing steeper hills. At 0.1m stride length, Cassie can climb up a hill with 32\% steeper incline. 
Overall, the task region gained is much bigger than the lost in simulation.

\begin{figure}[t!]
\edit{
 \centering
  \includegraphics[width=1.0\linewidth]{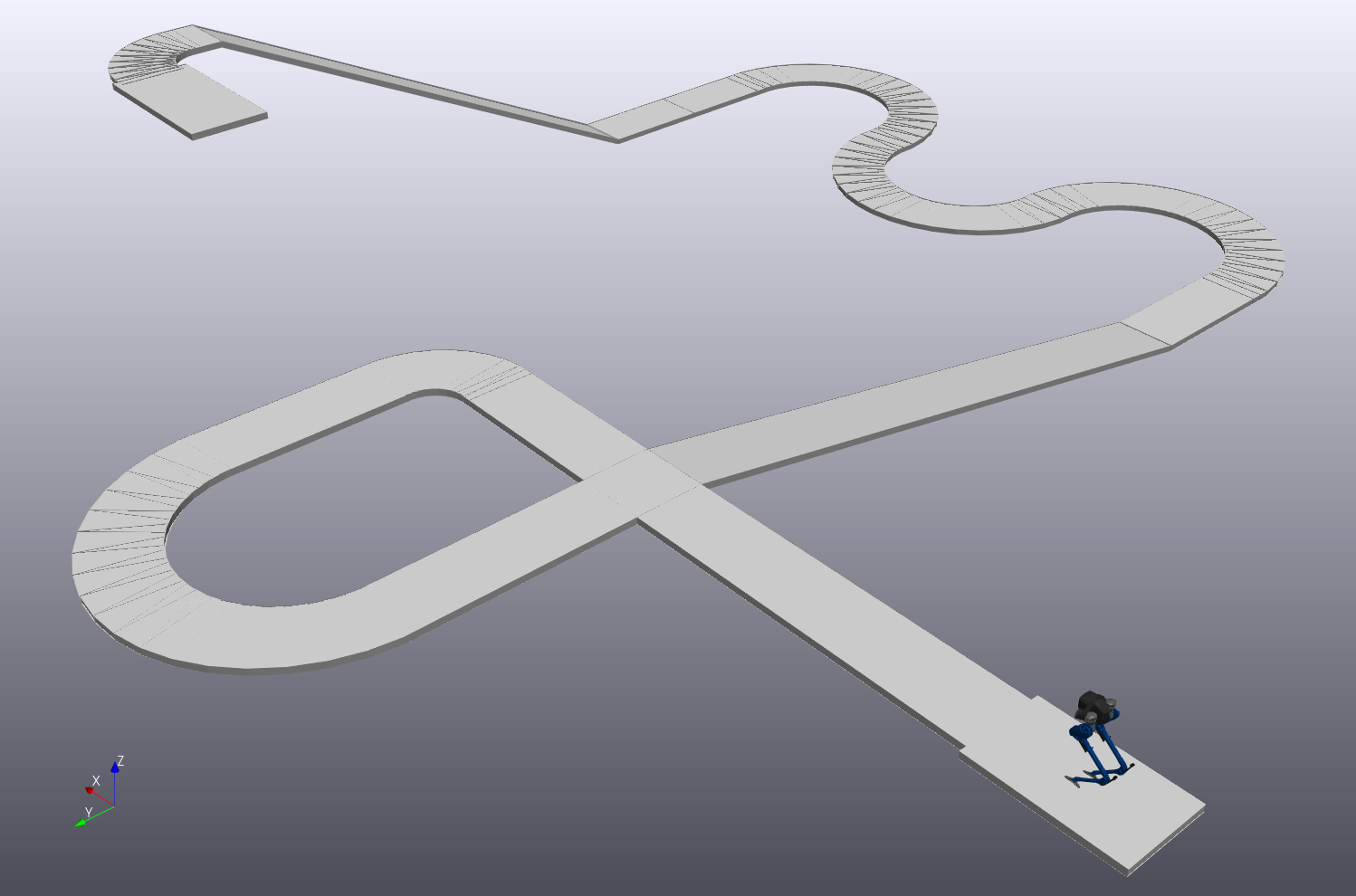}
 \vspace{-6mm}
 \caption{A track designed to showcase the performance difference between the LIP and the optimal ROM in simulation.
 The video of Cassie finishing the track can be found in the supplementary materials.
}
 \label{fig:track}
 }
\end{figure}

\begin{table}[t!]
\edit{
\centering
\begin{tabular}{ | c | c | c | c | } 
 \hline
 & LIP & Optimal ROM & Speed Improvement \\ 
 \hline
 Straight line (5m) & 5.72 (s) & 4.05 (s) & 41\% \\ 
 \hline
 Fast 90-degree turn & 1.65 (s) & 1.2 (s) & 38\% \\ 
 \hline
 Downhill (20\%) & 11.67 (s) & 8.4 (s) & 39\% \\ 
 \hline
 S-turns & 20.37 (s) & 14.47 (s) & 41\% \\ 
 \hline
 Uphill (50\%) & failed & 17.73 (s) & - \\ 
 \hline
\end{tabular}
\caption{Completion time for some segments of the course.}
\label{table:time_breakdown}
}
  \vspace{-3mm}
\end{table}

\begin{figure}
 \centering
 \includegraphics[width=.9\linewidth]{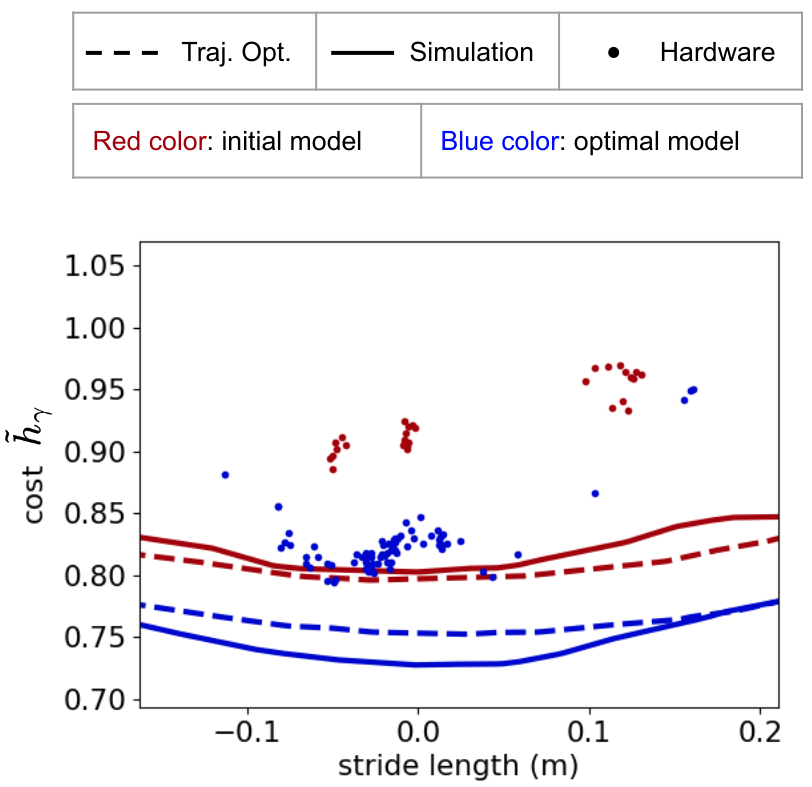}
 \vspace{-2mm}
 \edit{	
\caption{Cost comparison between the initial model \ref{case:initial_rom} and the optimal model \ref{case:optimal_rom}. 
The cost $\tilde{h}_\gamma$ is the cost function in Eq. \eqref{eq:model_trajopt}.
For trajectory optimization and simulation, we densely sampled the tasks and interpolated the costs.
 For the hardware experiment, we plot the costs of collected data points directly in the figure without interpolation. 
 To collect the data on hardware, we used a remote control to walk Cassie around, and we applied a moving window of 4 foot steps to extract periodic gaits according to Table \ref{table:periodic_criteria}. 
 We note that there was about 2 cm of height variation in the hardware experiment, but we normalized every extracted data point to the same height using the height-cost relationship from the trajectory optimization to make fair comparisons.
          }
         }
 \label{fig:cost_landscape_comparisons}
  \vspace{-0mm}
\end{figure}

Comparing the cost landscape between the trajectory optimization and simulation, we can see that the task region gained are similar\footnote{In the case of \ref{env:trajopt}, there was not a clear threshold for defining the failure of a task. We picked a cost threshold at which the robot's motion does not look abnormal.}. 
Cassie in general can walk faster at different turning rates and on different ground inclines.
We also observed that the cost landscapes of the open- and closed-loop share a similar profile in ground incline. 
Both show bigger cost reduction in walking downhill than uphill.
In contrast, the landscapes in turning rate look different between the open and closed loop. 
This is partially because there is only one stride (left support phase) in trajectory optimization while we average the cost over 4 strides in the simulation.
Additionally, there is a difference in stride lengths between the open- and closed-loop, showing a control challenge in stabilizing around high-speed nominal trajectories. 
This gap can be mitigated by increasing the control gains in simulation. 
However, we avoid using unrealistic high gains because they do not work on hardware.

\begin{figure}[!t]
\centering
\subfloat[Initial model \ref{case:initial_rom}. 2D slice at CoM position = 0m in y axis (when the CoM is right above the stance foot).]{
\includegraphics[width=0.85\linewidth]{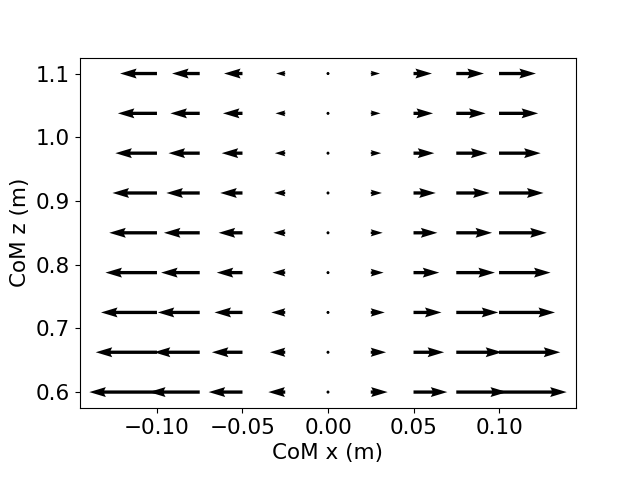}
\label{fig:com_dynamics_vector_field_iter1}}

\subfloat[Optimal model \ref{case:optimal_rom}. 2D slice at CoM position = 0m in y axis (when the CoM is right above the stance foot).]{
\includegraphics[width=0.85\linewidth]{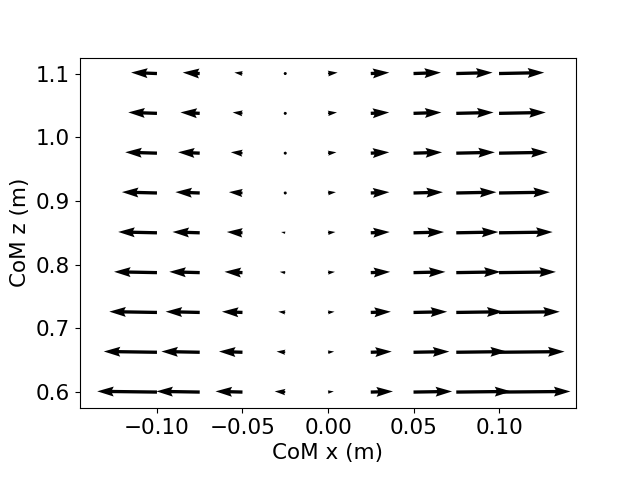}
\label{fig:com_dynamics_vector_field_iter400}}        

\subfloat[Optimal model \ref{case:optimal_rom}. 2D slice at CoM position = -0.2m in y axis (when the CoM is to the right of the left stance foot, for example).]{
\includegraphics[width=0.85\linewidth]{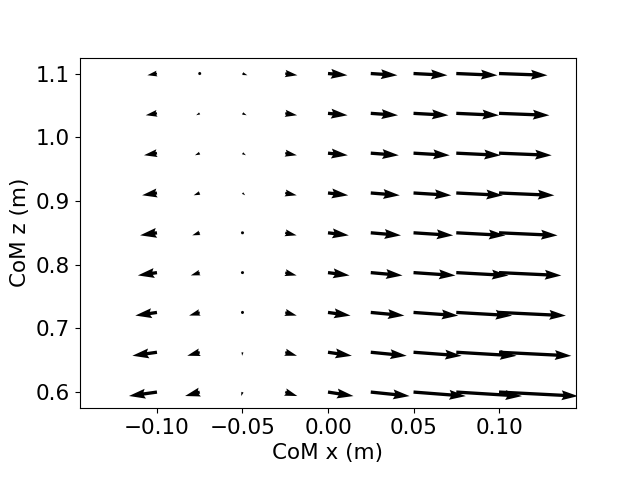}
\label{fig:com_dynamics_vector_field_iter400_different_slice}}        
		
\vspace*{5mm}

\caption{The vector fields of the ROM dynamics $g$ over the CoM x and z position. 
In this example, the dynamics is the acceleration of the CoM, which is a function of the CoM position and velocity defined in Eq. \eqref{eq:model_dyn}.
The first plot is the initial model's dynamics, while the latter two are that of the optimal model at two different slices of CoM y position. 
In all plots, the CoM velocity is 0.
We note that the size of the vectors only reflects the relative magnitude.
The absolute magnitude of the vectors for the first two plots are shown in Fig. \ref{fig:com_dynamics_magnitude}, although the scales of the x axis are different. 
}
\label{fig:com_dynamics_vector_field}
\end{figure}

\begin{figure}[!t]
\centering
\subfloat[Initial model \ref{case:initial_rom}. 2D slice at CoM position = 0m in y axis (when the CoM is right above the stance foot).]{
\includegraphics[width=1\linewidth]{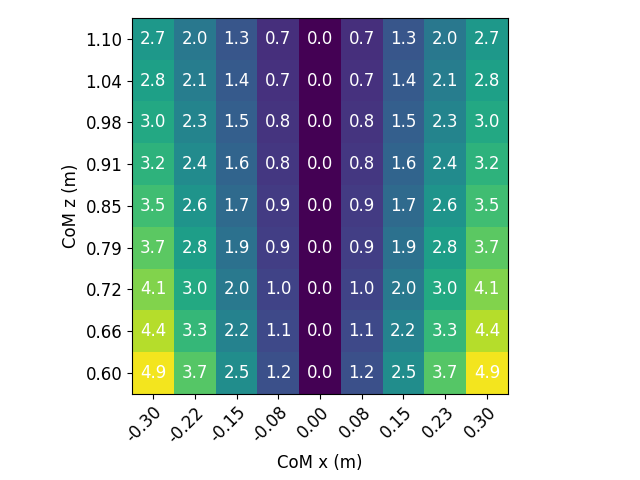}
\label{fig:com_dynamics_magnitude_field_iter1}}

\subfloat[Optimal model \ref{case:optimal_rom}. 2D slice at CoM position = 0m in y axis (when the CoM is right above the stance foot).]{
\includegraphics[width=1\linewidth]{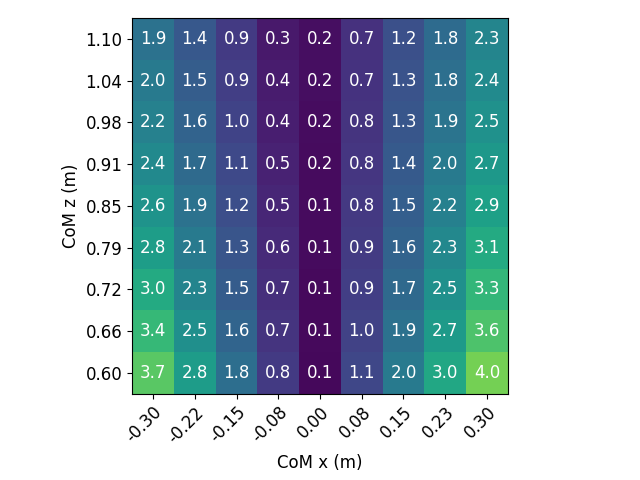}
\label{fig:com_dynamics_magnitude_field_iter400}}               
		
\vspace*{2mm}
\caption{The magnitude of ROM dynamics $g$ over the CoM x and z position. 
The settings of Fig. \ref{fig:com_dynamics_magnitude_field_iter1} and \ref{fig:com_dynamics_magnitude_field_iter400} are the same as Fig. \ref{fig:com_dynamics_vector_field_iter1} and \ref{fig:com_dynamics_vector_field_iter400}, respectively.
The magnitude plots show that the optimal model has smaller CoM accelerations, which implies smaller ground reaction forces, particularly in the x axis (implied by the vector fields in Fig. \ref{fig:com_dynamics_vector_field}).
We observed in the experiments that these force vectors align more closely with the normal direction of the ground.
}
\label{fig:com_dynamics_magnitude}
\end{figure}

In Fig. \ref{fig:cost_landscape_comparisons}, the dashed-line boxes represent the training task space used in the model optimization stage.
There is not a strong correlation between the cost ratio and the training space.
However, we can observe that the performance of the optimal ROM generalizes well to the unseen task (region outside the dashed-line box).
For example, the cost reduction ratio stays around 16\% at different turning rates in simulation.

Lastly, we designed a track shown in Fig. \ref{fig:track} for Cassie to finish as fast as possible to showcase the capability of the optimal ROM.
The track includes various segments requiring Cassie to turn by different angles and walk on different sloped grounds.
To enable Cassie to race through the track, we implemented a high-level path-following controller that sends commands (such as walking velocity) to the MPC.
We tuned the \eedit{path-following } controller parameters, so that Cassie can finish many segments as fast as possible without falling off the track. 
\eedit{The parameters were tuned separately for each model. 
We observe that, during high-speed walking, the optimal model can track the commanded velocity better than the LIP model, 
and that the actual top speed is achieved understand different commanded velocities. 
Quantitatively, Cassie on average can walk about 40\% faster with the optimal ROM, shown in Table \ref{table:time_breakdown}. }
This speed improvement reflected the periodic-walking result shown in the cost landscape plots in Fig. \ref{fig:cost_landscape_comparison_3d_task}.
Additionally, we observed that for the task of 50\% ground incline Cassie with LIP exhibited many stop-and-go motions and eventually fell, while Cassie with the optimal ROM was able to complete the incline steadily.

\subsection{Straight-line Walking on Hardware}
\label{sec:cost_comparison_for_trajopt_sim_and_hardware}

In this section, we aim to evaluate the model performance on the real robot.
We evaluate the cost for different stride lengths while fixing the pelvis height (0.95 m), and then plot the costs of both the initial and the optimal model directly in Fig. \ref{fig:cost_landscape_comparisons}. 
We also conducted the same experiment but in trajectory optimization and simulation for comparison.
Additionally, in order to maximize the controller robustness for the hardware experiment, we constrained the center of pressure (CoP) close to the foot center during the model optimization stage, although this limits the potential performance gain of the optimal ROM.

In Fig. \ref{fig:cost_landscape_comparisons}, we can see that the performance of the optimal ROM (evaluated by Cassie's joint torque squared in this case) is better than that of the initial ROM.
On hardware, the performance improvement is around 10\% for low-speed and medium-speed walking.
As a comparison, the cost improves by up to 8\% in open-loop and improves by up to 14\% in simulation using the same optimal ROM. 
We can also observe that the hardware costs are higher than those of open-loop and simulation, which might result from additional torques required to track the desired trajectories due to sim-to-real modeling errors.
Nonetheless, the improvement percentages are fairly similar across the board. 
This demonstrates that the model performance was successfully transferred to the hardware via the MPC, despite the modeling error in the full-order model.

\subsection{Optimal Robot Behaviors}\label{sec:optimal_robot_behavior}

In order to understand the source of the performance improvement, we look at the motion of the robot and the center of mass dynamics (the ROM dynamics $g$).
The discussions in this section are based on the straight-line walking experiments in Section \ref{sec:cost_comparison_for_trajopt_sim_and_hardware}.

We observe in the full-model trajectory optimization that the average CoP stays at the center of the support polygon when we use LIP on Cassie \ref{case:initial_rom}.
In contrast, the CoP moves toward the rear end of the support polygon where there is no ROM embedding \ref{case:no_rom}.
Interestingly, this CoP shift emerges when using the optimal model \ref{case:optimal_rom} in the open-loop trajectory, 
and we observe the same behavior in both simulation and hardware experiment.
On hardware, we confirm this by visualizing the projected CoM on the ground when Cassie walks in place.
The projected CoM is close to CoP, since there is little centroidal angular momentum for walking in place. 
The projected CoM of the hardware data indeed shifts towards the back of the support polygon when using the optimal model.

To understand why the projected CoM moves backward, we plot the ROM dynamic function $g$ in Fig. \ref{fig:com_dynamics_vector_field} for both the initial model and the optimal model. 
In the case of the initial model (LIP), we know that the dynamics should be symmetrical about the z axis, specifically the acceleration should be 0 at x = 0 (Fig. \ref{fig:com_dynamics_vector_field_iter1}).
This vector field profile, however, looks different in the case of the optimal model.
As we can see from Fig. \ref{fig:com_dynamics_vector_field_iter400}, the area with near-zero acceleration shifts towards the -x direction (i.e. to the back of the support polygon), and interestingly it also slightly correlates to the height of the CoM.
The higher the CoM is, the further back the region of zero acceleration is.
Additionally, we know that the LIP dynamics in the x-z plane is independent of the CoM y position.
That is, no matter which slice of the x-z plane we take, the vector field should look identical.
For the optimal model, the dynamics in the x-z plane is a function of the CoM y position.
The further away the CoM is from the stance foot (bigger foot spread), the further back the zero acceleration region is.

Aside from the vector field plots of the CoM dynamics, we also visualize the absolute value of the vectors in Fig. \ref{fig:com_dynamics_magnitude}. 
We can see that the magnitude of the CoM acceleration in general becomes smaller when using the optimal model.
This implies that the total ground reaction force is smaller with the optimal model, even if the robot walks at the same speed.
Given the same walking speed, the robot decelerates and accelerates less in the x axis when using the optimal model (i.e., the average speed is the same, but the speed fluctuation becomes smaller after using the optimal model).
We hypothesize that the decrease in ground force magnitude partially contributes to the decrease in the joint torque in the case of Cassie walking. 
The fact that there is less work done on the CoM during walking with the optimal ROM might have led to the decrease in torque squared which is a proxy for energy consumption.

The experiments in Section \ref{sec:closed_loop_eval} demonstrate two things. 
First, the optimal behavior and the performance are transferred from the open-loop training (left side of Fig. \ref{fig:outline}) to a closed-loop system (right side of Fig. \ref{fig:outline}).
Second, the optimal reduced-order model improves the real Cassie's performance, while the low dimensionality permits a real time planning application.

\edit{

}

\section{MPC for a General ROM}\label{sec:mpc_generic_rom}

\edit{

Sections \ref{sec:mpc} and \ref{sec:closed_loop_eval} limit the ROMs to a predefined embedding function $r$ which simplifies the planner and enables real time planning results. 
In this section, we present an MPC formulation for a general ROM, 
where full-order states at swing foot touchdown events were used to provide physical meaning to the resulting plan.
Given an ongoing steady improvement in computational and algorithmic speed, we believe this general MPC will soon be solvable in real time on hardware. 

}

\subsection{Hybrid Nature of the Robot Dynamics}\label{sec:hybrid_discussion}

Shown in Eq. \eqref{eq:hybrid_eom}, the dynamics of the full model is hybrid -- it contains both the continuous-time dynamics and discrete-time dynamics due to the foot collision.
In contrast, many existing reduced-order models assume \edit{zero ground impacts }at foot touchdown.
This is partially due to the fact that the exact embedding of a reduced-order discrete dynamics does not always exist.
For example, we could have two pre-impact states $x$ of the full model that correspond to the same reduced-order states, but the post-impact states of the full model map to two different reduced-order states.
In this case, the reduced-order discrete dynamics is not an ordinary (single-valued) function.
Therefore, in order to capture the exact full impact dynamics in the planner, it is necessary to mix the reduced-order model with the discrete dynamics from the full-order model.  
We note that the traditional approaches to reduced-order planning and embedding must also grapple with approximations of the impact event.

In addition to the issue above, the mix of reduced- and full-order models also seems necessary if we do not retain the physical interpretability of the embedding $r$ when planning for the optimal footstep locations in the planner.
This results in a low-dimensional trajectory optimization problem, a search for $y_j(t)$ and $\tau_j(t)$, with additional decision variables $x_{-,j},x_{+,j}$, representing the pre- and post-impact full-order states.
The index $j$ refers to the $j$-th stride. 
The constraints relating the reduced-order state to the full-order model and the impact dynamics are \vspace{-2mm}
\begin{equation} \label{eq:hybrid_constraint}
\begin{aligned}
y_j(t_j) = r(q_{-,j}; \theta_e)&, & \dot{y}_j(t_j) =\frac{\partial r (q_{-,j};\theta_e)}{\partial q_{-,j}} \dot{q}_{-,j},\\
y_{j+1}(t_j) = r(q_{+,j}; \theta_e)&, & \dot{y}_{j+1}(t_j) =\frac{\partial r (q_{+,j};\theta_e)}{\partial q_{+,j}} \dot{q}_{+,j},\\
\text{and } & &C_{hybrid}(x_{-,j}, x_{+,j},\Lambda_j) \leq 0,
\end{aligned}
\end{equation}
where $t_j$ is the impact time (ending the $j$-th stride), $C_{hybrid}$ represents the hybrid guard $S$ and the impact mapping $\Delta$ without left-right leg alternation\footnote{The impact mapping $\Delta$ can be simplified to identity if we assume no impact (i.e. swing foot touchdown velocity is 0 in the vertical axis).} \cite{westervelt2003hybrid}.

\subsection{Planning with ROM and Full-order Impact Dynamics}\label{sec:plan_with_rom_and_fom}

Similar to Section \ref{sec:mpc}, we formulate a reduced-order trajectory optimization problem to walk $n_s$ strides.
However, we replace the discrete footstep inputs $\Tau_{fp}$ with the full robot states $x_{-,j}$ and $x_{+,j}$.
\edit{
The key difference between the planning problems in Section \ref{sec:mpc} and this section is that here we introduce new variables 
\begin{itemize}
    \item full-order states $x_{-,j}$, $x_{+,j}$ and ground impulses $\Lambda_j$ (to capture the exact full-order impact dynamics), and 
    \item the ROM’s continuous-time input $\tau$.
\end{itemize}
}
To improve readability, we stack decision variables into bigger vectors $\Tau=[\tau_1, ..., \tau_n] \in \mathbb{R}^{n_\tau n}$ and $X=[x_{-,1},...,x_{-,n_s}, x_{+,1},...,x_{+,n_s}] \in \mathbb{R}^{2n_x n_s}$.

Costs are nominally expressed in terms of $[y, \dot{y}]$ and $\tau$, though the pre- and post-impact full-order states can also be used to represent goal locations.
In addition to the constraints in Eq. \eqref{eq:hybrid_constraint}, we impose constraints $C_{kinematics} (X) \leq 0$ on the full model's kinematics such that the solution obeys joint limits, stance foot stays fixed during the stance phase, and legs do not collide with each other.

The planning problem with the general ROM is 
\begin{equation}\label{eq:rom_planning_trajopt_mixed_with_FoM}
    \begin{array}{cl}
     \underset{w}{\text{min}} & \| \Tau \|^2_{W_\Tau} +  \| Z - Z_{reg} \|^2_{W_Z} +  \| X - X_{reg} \|^2_{W_X}  \\
     \text{s.t.} & \text{Reduced-order dynamics (Eq. \eqref{eq:model_dyn})}, \\
      & \text{Hybrid constraints (Eq. \eqref{eq:hybrid_constraint}}),\\
      & C_{kinematics} (X) \leq 0, \\
      & x_0 = \text{current feedback full-order state},  \\
    \end{array}
\end{equation}
where $w = [Z, \Tau, x_0, X, \Lambda_1, ..., \Lambda_{n_s}] \in \mathbb{R}^{n_w}$, $W_T$, $W_Z$ and $W_X$ are the weights of the norms, $Z_{reg}$ is the regularization state for the reduced model, and $X_{reg}$ is the regularization state for the full state and contains the goal location of the robot.
After solving Eq. \eqref{eq:rom_planning_trajopt_mixed_with_FoM}, we reconstruct the desired ROM trajectory $y_d(t)$ from the optimal solution $Z^*$. 
Different from Section \ref{sec:mpc}, the optimal solution $w^*$ here also contains the desired full-order states $X^*$ at impact events, from which we derive not only the desired swing foot stepping locations but also the desired trajectories for joints such as swing hip yaw and swing toe joint.
Additionally, since there are full states in the planner, we can send the turning rate command directly to the ROM planner.

Fig. \ref{fig:rom_planning} visualizes the pre-impact states in the case where the robot walks two meters with four strides, connected by the hybrid events and continuous low-dimensional trajectories $y_j(t)$. 
Although there is no guarantee that the planned trajectories $y_j(t)$ are feasible for the full model except those at the hybrid events, we were able to retrieve $q(t)$ from $y_j(t)$ through inverse kinematics, meaning the embedding existed empirically.
We note that classical models like LIP also provide no guarantees \cite{iqbal2022drs}.
For instance, there is no constraint on leg lengths in the ROM which could lead to kinematic infeasibility.
The formulation in Eq. \eqref{eq:rom_planning_trajopt_mixed_with_FoM} preserves an exact representation of the hybrid dynamics, but results in a significantly reduced optimization problem.

\begin{figure}[t!]
 \centering
 \includegraphics[width=0.9\linewidth]{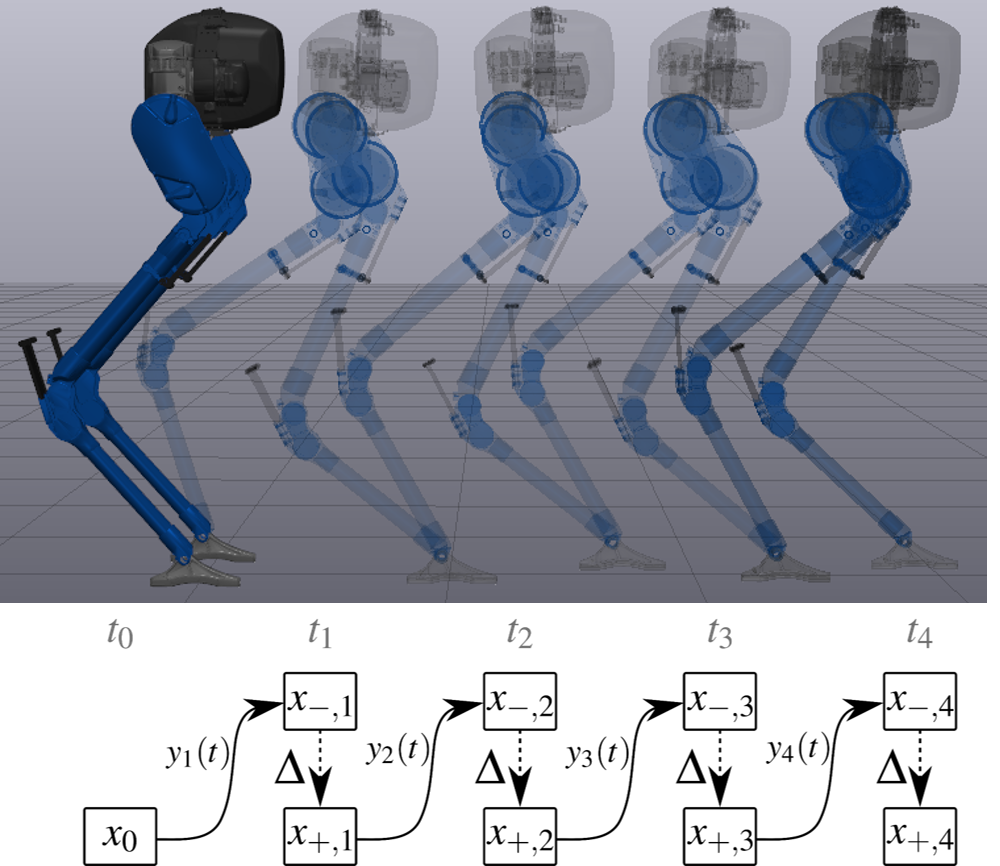}
 \caption{An example of the real time planner in Eq. \eqref{eq:rom_planning_trajopt_mixed_with_FoM}. Given a task of covering two meters in four steps starting from a standing pose, we rapidly plan a trajectory for the reduced-order model. The high-dimensional model is used to capture the hybrid event at stepping, as illustrated in the diagram.}
 \label{fig:rom_planning}
\end{figure}

\subsection{Implementation and Experiments}

We implement the MPC using Eq. \eqref{eq:rom_planning_trajopt_mixed_with_FoM} for both simulation and hardware experiments (the hardware setup is the same as Section \ref{sec:hardware_setup_and_solve_time}). 
In simulation, we were able to transfer the open-loop performance to closed-loop performance with this new MPC. 
However, on hardware, the off-the-shelf solvers IPOPT and SNOPT were not capable of solving the planning problem in Eq. \eqref{eq:rom_planning_trajopt_mixed_with_FoM} fast enough or well enough to enable a high-performance real time MPC.
With IPOPT, the planner simply did not run fast enough.
With SNOPT, even though the solve time can be decreased down to 30ms with loose optimality tolerance and constraint tolerance, we sacrificed the solution quality too much to achieve big stride length on Cassie.
\edit{
Nonetheless, we believe that the general MPC will soon be solvable within reasonable time constraints on hardware, as computer technology advances steadily.
Additionally, a well-engineered custom solver can also help enable real time planning.
Boston Dynamics has shown their success in the nonlinear MPC solver with the centroidal momentum model and full model configurations\footnote{
The centroidal momentum model \cite{orin2008centroidal, wensing2016improved} describes the actual dynamics without imposing constraints on the full model.
However, this momentum model lacks positional representations, thus requiring the incorporation of full model configurations in planning.
} \cite{BostonDynamicsTalkAtWorkship}.
}

\section{Discussion}\label{sec:discussion}

\subsection{Model Parameterization}\label{sec:parameterization_discussion}

\subsubsection{Trade-off between planning speed and performance}

\eedit{
A fundamental trade-off exists between planning speed and model performance \cite{li2021model, norby2022adaptive}.
As the model (i.e. $r$ and $g$) becomes more expressive, we see slower planning speeds and better performances.
This model expressiveness increases as we 
\begin{itemize}
\item increase the dimension of the model,
\item increase the number of basis functions (or number of neurons in a neural network), or
\item use universal function approximators such as monomial basis functions (polynomials) instead of physically-interpretable functions (particular set of functions).
\end{itemize}
Additionally, under the same model expressiveness, we observe a trade-off between the task space size and the model performance.
Larger task spaces require the model to be less specialized, which might result in a lower average performance.
}

\eedit{
\subsubsection{Linear models (with linear basis functions)}
We found that, for some choice of task space, linear models 
perform almost as well as the full model. 
}
This linear reduced-order dynamics transforms the MPC in Eq. \eqref{eq:rom_planning_trajopt} into a quadratic optimization problem, allowing for rapid planning.
This linear model also renders a closed-form solution and makes it suitable for existing techniques in robust control design and stability analysis.
For challenging or complex task spaces, linear basis functions sacrifice significant performance when compared with those of higher degree. 
We emphasize that our method can be used to optimize models of any chosen degree, and leave such selection to the practitioner.

\subsubsection{Alternative basis functions}
Beside different orders of monomials, we also experimented with trigonometric monomials (e.g. $\text{sin}^a(x) \text{cos}^b(x)$ where $a, b \in \mathbb{N}$).
However, we found no notable difference with this basis set.
Since quadratic basis approaches optimal performance in model optimization in Section \ref{sec:model_optimization_examples}, 
we leave a broader exploration of choices of basis functions as a possible future work.

\subsubsection{Physical Interpretability of ROMs}

Classical ROMs often maintain some level of physical interperability, because they are built from mechanical components like springs and masses. 
Our approach, which uses more general representations, does sacrifice this connection to human intuition. 
However, we have found it beneficial to manually select the embedding function $r$. 
This has the benefit of ensuring that the reduced-order state remains human interpretable, which is  useful for specifying objectives for planning and control in Section \ref{sec:mpc}. 
One might also imagine restricting the space of ROM dynamics $g$ to maintain physical connections (e.g. specifying nonlinear or velocity-dependent springs, inspired by human studies \cite{pequera2023reducing}), 
though we leave this to future work.

\subsection{Performance Gap Between Open-loop and Closed-loop}\label{sec:performance_gap}

The proposed approach to model optimization uses full-model trajectory optimization. 
This has a few advantages. 
First,  it allows us to embed the reduced-order model into the full model exactly via constraints. 
Second, it is more sample efficient than the approaches in reinforcement learning (such as \cite{pandala2022robust}).
However, using trajectory optimization leaves a potential performance gap between the offline training and online deployment, because trajectory optimization is an open-loop optimal control method which does not consider any controller heuristics by default.
For instance, our walking controller constructs the swing foot trajectory with cubic splines, and we found the open-loop performance can be transferred to the closed-loop better after we add this cubic spline heuristic as a constraint in the trajectory optimization. 

While we have seen the performance of the robot improve, we also observed that the solver would exploit any degree of freedom of the input and state variables during the model training stage.
For example, the CoP turned out to play an important role in improving the performance (reducing joint torques) of Cassie in Section \ref{sec:closed_loop_eval}.
If we had not regularized the CoP (or the CoM motion) during the model optimization stage, the solver would have moved the CoP all the way to the edge of the support polygon.
Although this exploitation can potentially lead to a much bigger cost improvement,
it also hurts the robustness.
Under sensor noises and model uncertainty, tracking the planned trajectories of this optimal model cannot stabilize the robot well, and thus the performance cannot be transferred to the hardware. 
One principled way of fixing this issue is to optimize the robustness of the trajectory alongside the user-specified cost function \cite{Dai12a, zhu2022hybrid}.

For the general optimal ROM MPC in Section \ref{sec:mpc_generic_rom}, one place where the performance gap can enter is the choice of the cost function in Eq. \eqref{eq:rom_planning_trajopt_mixed_with_FoM}.
In our past experiments, we simply ran a full model trajectory optimization and used this optimal trajectory for the regularization term in Eq. \eqref{eq:rom_planning_trajopt_mixed_with_FoM}.
It worked well, although there was a slight improvement drop. 
To mitigate the gap, one could try inverse optimal control to learn the MPC cost function given data from the full model trajectory optimization.

Since our robot has one-DOF underactuation caused by line feet design, it is not trivial to track the desired trajectory of the ROM within the continuous phase of hybrid dynamics.
We noticed in the experiment that there was a noticeable performance difference between whether or not we tracked the first element of the desired ROM trajectory. 
Therefore, we conjecture that if we use a robot with a finite size of feet, we could translate the open-loop performance to the closed-loop performance better and more easily. 

Lastly, we found that empirically it is easier to transfer model performance from open-loop to closed-loop when we fix the embedding function $r$, although the open-loop performance improvement is usually much bigger (sometimes near full-model's performance) when we optimize both the embedding function $r$ and dynamics $g$.
As a concrete example, in some model optimization instances the optimal ROM position $y$ can be insensitive to the change of CoM position, which makes it difficult to servo the CoM height. 
In the worst case, this insensitivity could lead to substantial CoM height movement and instability of the closed-loop system. 

\vspace{-2mm}

\subsection{Limitations of Our Framework} \label{sec:limitation}
In our bilevel optimization approach, the initial ROM must be feasible for the inner-level trajectory optimization to obtain a meaningful gradient for the outer-level optimization.
This means that we must initialize the ROM to one capable of walking, potentially limiting our ability to use stochastic initialization to explore the entire ROM space.
Despite this potential drawback, we note that random task sampling in Algorithm \ref{alg:model_opt} can help escape certain local minima (effectiveness depends on the cost landscape of the model optimization).
Future work could explore the role of this initialization, for instance by evaluating performance when starting from multiple existing hand-designed ROMs.

Our approach requires the user to determine the dimension of the ROM. 
Increasing the dimension theoretically strictly improves model performance, at the cost of MPC computational speed. 
As a result, this defines a Pareto optimal front, without a simple way to automatically determine the dimension. 
That said, there are recent works which attempt to select between models of varying complexity \cite{li2021model, khazoom2023optimal, norby2022adaptive}, 
which we believe might be applied to our framework.

\vspace{-2mm}

\subsection{Generality of Our Framework}

\edit{
This paper focuses on applying the optimal ROMs to the hardware Cassie, 
but throughout the project we observe that LIP performs reasonable well for Cassie, 
particularly over relatively simple task domains such as straight-line walking. 
We hypothesize that this is due, in part, to the fact that Cassie's legs are relatively light. 
As an experiment, we investigated the effect of foot weight on the performance improvement.
When increasing the foot's mass to 4kg (the robot weighs 40kg in total), 
we observed that the LIP cost relative to the full model's increased from 1.3 to 1.8 and offered a greater room for improvement, resulting in 40\% torque cost reduction for tasks similar to Example \exDefault's.
Beside this investigation, we also saw more than 75\% of cost reduction for the five-link planar robot in our prior work \cite{chen2020optimal}. }

Furthermore, the proposed framework is agnostic to types of robots and tasks \edit{(e.g. quadrupeds and dexterous manipulators). }
This has implications all over robotics, given the need for computational efficiency and the prevalence of reduced-order models in locomotion and manipulation.

\vspace{-1mm}

\section{Conclusion}\label{sec:conclusion}
In this work, we directly optimized the reduced-order models which can be used in an online planner that achieved performance higher than that of the traditional physical models. 
We formulated a bilevel optimization problem and presented an efficient algorithm that leverages the problem structure.
Examples showed improvements up to 38\% depending on the task difficulty and the performance metric.
The optimal reduced-order models are more permissive and capable of higher performance, while remaining low dimensional.
We also designed two MPCs for the optimal reduced-order models which enable Cassie to accomplish tasks with better performance.
In the hardware experiment, the optimal ROM showed 10\% of improvement on Cassie, and we investigated the source of performance gains for this particular model.
We demonstrated that the use of ROM greatly reduces planning time, and that the optimized ROM improves the performance of the robot beyond the traditional ROMs.

Although the model optimization approach presented in this paper has the advantage of optimizing models agnostic to low-level controllers, it does not guarantee that the performance improvements from these optimal models can be transferred to the robot via a feedback controller as discussed in Section \ref{sec:performance_gap}.
One ongoing work is to fix the above issue by optimizing the model in a closed-loop fashion, so that the model optimization accounts for the controller heuristics and maintains the closed-loop stability. 
This would also potentially ease the process of realizing the optimal model performance on hardware.
Additionally, discussed in Section \ref{sec:hybrid_discussion}, an approximation is necessary if we were to find a low dimensional representation of the full-order impact dynamics.
In this work, we only circumvented the hybrid problem by either using a physically-interpretable ROM or mixing the full impact dynamics with the ROM. 
Finding an optimal low-dimensional discrete dynamics for a robot still remains an open question.

\section*{Acknowledgements}
We thank Wanxin Jin for discussions on Envelope Theorem and approaches to differentiating an optimization problem. Toyota Research Institute provided funds to support this work.

\appendices

\section{Heuristics in trajectory optimization} \label{sec:heuristics_in_trajopt}

Solving the trajectory optimization problem in Eq. \eqref{eq:disc_trajopt} or \eqref{eq:model_trajopt} for a high-dimensional robot is hard, since the problem is nonlinear and of large scale.
Even although there are off-the-shelf solvers such as IPOPT \cite{wachter2006implementation} and SNOPT \cite{gill2005snopt} designed to solve large-scale nonlinear optimization problem, it is often impossible to get a good optimal solution without any heuristics\edit{, since there are many local optima. }
In this section, we will talk from our experience about the heuristics that might help to solve the problem faster and also \edit{find a solution with a lower cost and closer to the global optimum.
That said, we have no objective manner in which to assess proximity to global optimality, and thus this is a purely observational criterion.
}

Let the nonlinear problem be 
\vspace{-1mm}
\begin{equation}\label{eq:large_scale_nlp}
    \begin{array}{cl}
     \underset{w}{\text{min}} & \tilde{h}(w)  \\
     \text{s.t.} & \tilde{f}(w) \leq 0\\
    \end{array}
\vspace{-2mm}
\end{equation}
where $w$ contains all decision variables, $\tilde{h}$ is the cost function, and $\tilde{f}$ is the constraint function. 
It turned out that scaling either $w$, $\tilde{h}$ or $\tilde{f}$ could help to improve the condition of the problem.

\begin{itemize}
  \item $w$: 
  Sometimes the decision variables are in different units and can take values of different orders. 
For example, joint angles of Cassie are roughly less than $1$ (rad), while its contact forces are usually larger than $100$ (N).
In this case, we can scale $w$ by some factor $s$, such that the decision variables of the new problem are $w_{scaled,i} = s_i w_i$ for $i = 1, 2, ..., n_w$. 
After the problem is solved, we scale the optimal solution of the new problem back by $w^*_i= \frac{1}{s_i} w^{*}_{scaled,i}$.

  \item $\tilde{f}$:
The constraints $\tilde{f}$ can take various units just like $w$. 
Similarly, we can scale each constraint individually.
Note that scaling constraints affects how well the original constraints are satisfied, so one should make sure that the constraint tolerance is still meaningful.   
  
  \item $\tilde{h}$:
  In theory, scaling the cost does not affect the optimal solution.
However, it does matter in the solver's algorithm.
It is desirable to scale the cost so that it is not larger than 1 around the area of interest.
\end{itemize}
For more detail about scaling, we refer the readers to Chapter 8.4 and Chapter 8.7 of \cite{gill1981practical}.
In addition to scaling the problem, the following heuristics could also be helpful:

\begin{itemize}
\item Provide the solver with a good initial guess.
\item Add small randomness to the initial guess: 
This helps to avoid singularities.
\item Add regularization terms to the cost function:
This could remove local minima in the cost landscape and can also speed up the solve time.
Adding regularization terms is similar to the traditional reward shaping of Reinforcement Learning \cite{hu2020learning} and the policy-regularized MPC \cite{bledt2017policy}.
\item Add intermediate variables (also called slack variables \cite{hereid2017frost}): 
This can sometimes improve the condition number of the constraint gradients with respect to decision variables.
One example of this is reformulating the trajectory optimization problem based on the single shooting method into that based on the multiple shooting method by introducing state variables \cite{betts1998survey}.
\item Use solver's internal scaling option:
In the case of SNOPT \cite{philip2015user}, we found setting \emph{Scale option} to 2 helps to find an optimal solution of better quality.
Note that this option increases the solve time and demands a good initial guess to the problem.
\end{itemize}

\section{Mirrored Reduced-order Model} \label{sec:mirror_rom_appendix}

\subsection{Model definition}

The model representation in Eq. \eqref{eq:model} could be dependent on the side of the robot.
For example, when using the LIP model as the ROM, we might choose the generalized position of the model $y$ to be the CoM position relative to the left foot (instead of the right foot). 
In this case, we need to find the reduced-order model for the right support phase of the robot.
Fortunately, we can derive this ROM by mirroring the robot configuration $q$ about its sagittal plane (Fig. \ref{fig:mirror_robot}) and reusing the ROM of the left support phase.
We refer to this new ROM as the \emph{mirrored reduced-order model}.

Let $q_m$ and $v_m$ be the generalized position and velocity of the ``mirrored robot", shown in Fig. \ref{fig:mirror_robot}. 
Mathematically, the mirrored ROM $\mu_m$ is 
\vspace{-2mm}
\begin{equation}\label{eq:mirrored_model_def}
\begin{aligned}
\mu_m & \triangleq (r_m,g),
\end{aligned}
\vspace{-2mm}
\end{equation}
with 
\vspace{-2mm}
\begin{subequations}\label{eq:mirrored_model}
\begin{align}
y_m&= r_m(q) = r(q_m) = r(M(q)) ,  \label{eq:mirrored_model_kin}\\
\ddot{y}_m&= g(y_m,\dot{y}_m,\tau), \label{eq:mirrored_model_dyn}
\end{align}
\end{subequations}
where $r_m$ is the embedding function of the mirror model, and $M$ is the mirror function such that $q_m = M(q)$ and $q = M(q_m)$. 
We note that the two models, in Eq.  \eqref{eq:model_def} and  \eqref{eq:mirrored_model_def}, share the same dynamics function $g$.

\subsection{Time derivatives of the embedding function}

Feedback control around a desired trajectory often requires the first and the second time derivatives information.
Here, we derive these quantities for the mirrored ROM in terms of the original embedding function $r$ in Eq. \eqref{eq:model_kin} and its derivatives.

\begin{figure}[t!]
 \centering
 \includegraphics[width=0.8\linewidth]{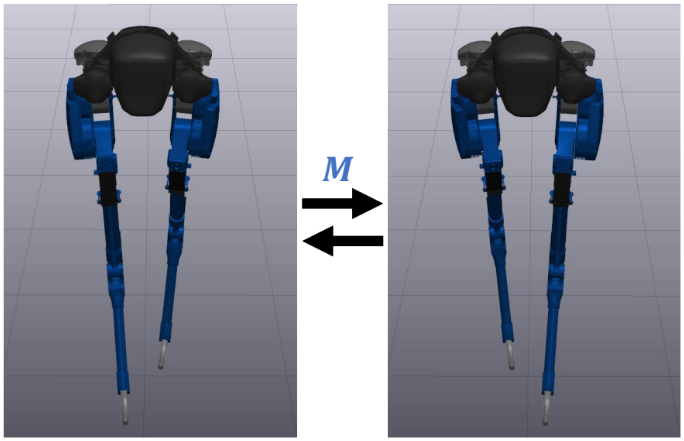}
 \vspace{-2mm}
 \caption{The mirror function $M$ mirrors the robot configuration about the sagittal plane. This function is necessary in planning and control when the embedding function (Eq. \eqref{eq:model_kin}) of the reduced-order model only represents one side of the robot. For example, the embedding function of the LIP model was chosen to be the CoM relative to the left foot (and not the right foot).}
 \label{fig:mirror_robot}
\end{figure}

\subsubsection{$\dot{y}_m$}

Let $J_m$ be the Jacobian of the mirrored model embedding $r_m$ with respect to the robot configuration $q$, such that 
\begin{equation}\label{eq:J_m_def}
\dot{y}_m = J_m \dot{q}.
\end{equation}
The time derivatives of $y_m$ is
\begin{equation}\label{eq:qdot_m}
\dot{y}_m =  \frac{\partial r(M(q))}{\partial M(q)} \frac{\partial M(q)}{\partial q} \dot{q} = J(q_m) \frac{\partial M(q)}{\partial q} \dot{q}  = J(q_m) \dot{q}_m
\end{equation}
where $J(q_m)$ is the Jacobian of the original model embedding $r$ evaluated with $q_m$.  
From Eq. \eqref{eq:J_m_def} and \eqref{eq:qdot_m}, we derive
\begin{equation} \label{eq:J_m}
J_m = J(q_m) \frac{\partial M(q)}{\partial q} 
\vspace{-2mm}
\end{equation}
where $\frac{\partial M(q)}{\partial q}$ is a matrix which contains only 0, 1, and -1.

\subsubsection{$\ddot{y}_m$}

The $i$-th element of $\ddot{y}_m$ is
\begin{equation*}
\begin{aligned}
\ddot{y}_{m,i}  & = \frac{d}{d t} \left( \frac{\partial r_i(q_m)}{\partial q_m} \dot{q}_m \right) 
= \frac{d}{d t} \sum_j \frac{\partial r_i(q_{m})}{\partial q_{m,j}} \dot{q}_{m,j}  \\
& = \sum_j \frac{d}{d t} \left( \frac{\partial r_i(q_m)}{\partial q_{m,j}} \right) \dot{q}_{m,j} + \sum_j \frac{\partial r_i(q_m)}{\partial q_{m,j}} \frac{d}{d t} \left( \dot{q}_{m,j} \right)  \\
& = \sum_{jk}   \frac{\partial^2 r_i(q_m)}{\partial q_{m,j}\partial q_{m,k}}  \dot{q}_{m,j} \dot{q}_{m,k} + \frac{\partial r_i(q_m)}{\partial q_{m}} \frac{d}{d t} \left( \dot{q}_{m} \right).  
\end{aligned}
\end{equation*}
where $r_i$ is the $i$-th element of the embedding function. The above equation can be expressed in the vector-matrix form
\begin{equation}\label{eq:qddot_m}
\begin{aligned}
\ddot{y}_{m}  
& = \dot{q}_{m}^T \nabla^2 r(q_m) \dot{q}_m + J(q_m) \ddot{q}_m \\
& = \dot{J}(q_m, v_m) \dot{q}_m + J(q_m)\ddot{q}_m \\
& = \dot{J}(q_m, v_m) \dot{q}_m + J_m \dot{v} \ \ \ \ (\because \text{Eq. \eqref{eq:J_m}})
\end{aligned}
\end{equation}
where $\dot{J}(q_m, v_m)$ is the time derivatives of the $J$ (of the original model) evaluated with the mirrored position $q_m$ and velocity $v_m$.

\bibliographystyle{ieeetr}
\bibliography{library,yuming_library}

\begin{IEEEbiography}[{\includegraphics[width=1in,height=1.25in,keepaspectratio]{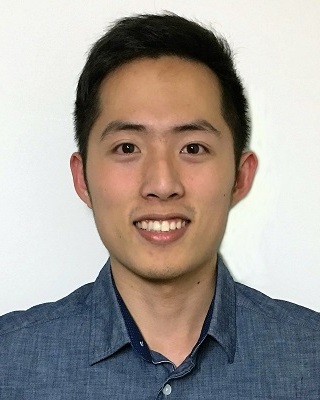}}]
{Yu-Ming Chen} is currently an applied scientist at The AI Institute. 
His research focuses on reduced-order modeling of legged robots with applications in real time planning and control. 
He holds a B.Sc. in Physics from National Taiwan University (2015), an M.Sc. in Robotics from the University of Michigan, Ann Arbor (2018) and a Ph.D. in Electrical and Systems Engineering from the University of Pennsylvania (2023).
\end{IEEEbiography}

\begin{IEEEbiography}[{\includegraphics[width=1in,height=1.25in,keepaspectratio]{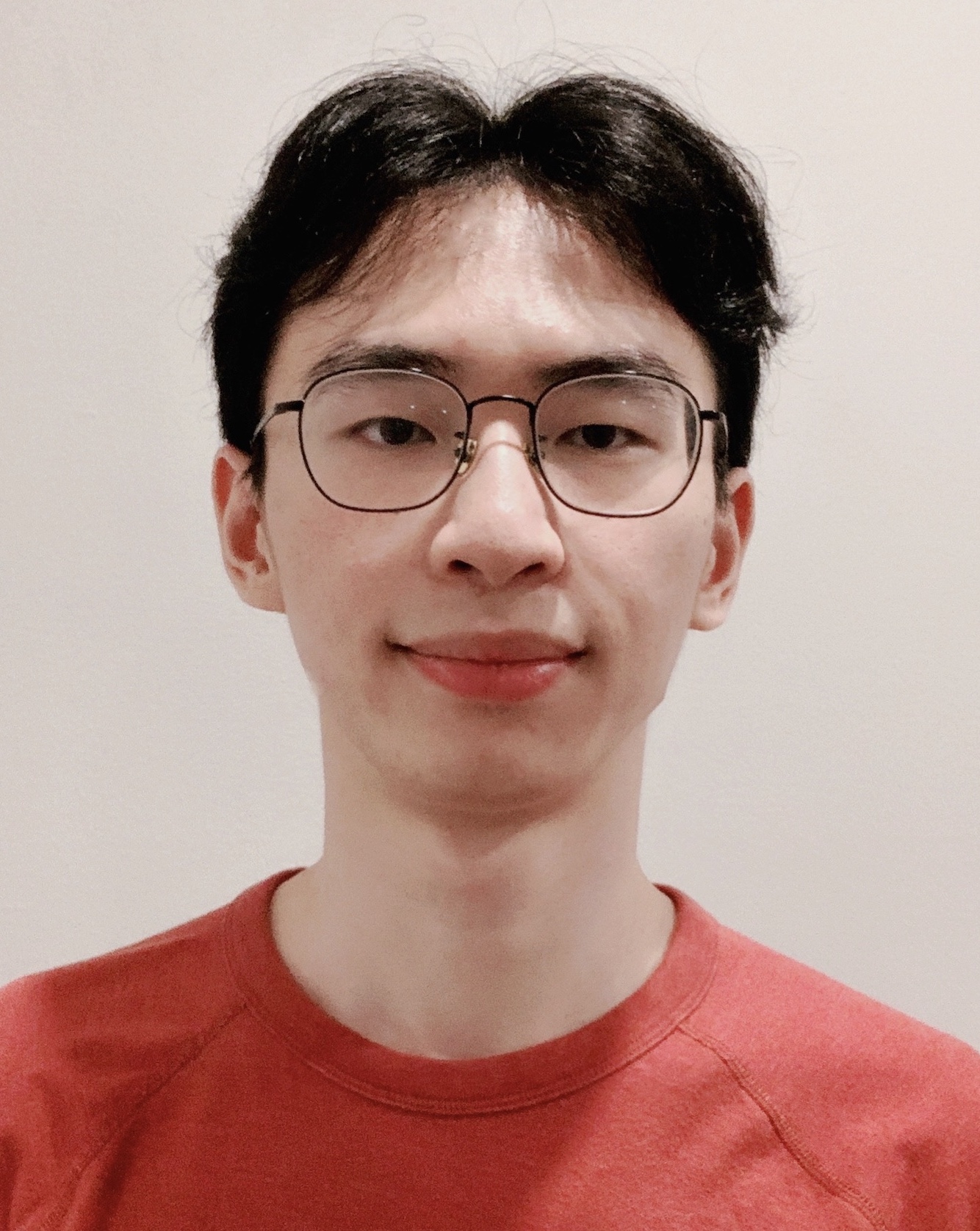}}]
{Jianshu Hu} is currently a PhD student in UM-SJTU Joint Institute, Shanghai Jiao Tong University. He is working on Deep Reinforcement Learning supervised by Professor Yutong Ban and Professor Paul Weng. Aiming at building an intelligent robot, he is focusing on designing sample-efficient Reinforcement Learning algorithms leveraging invariance/equivariance and data augmentation. Please check the personal website (https://jianshu-hu.github.io) for more details about his research.
\end{IEEEbiography}

\begin{IEEEbiography}[{\includegraphics[width=1in,height=1.25in,keepaspectratio]{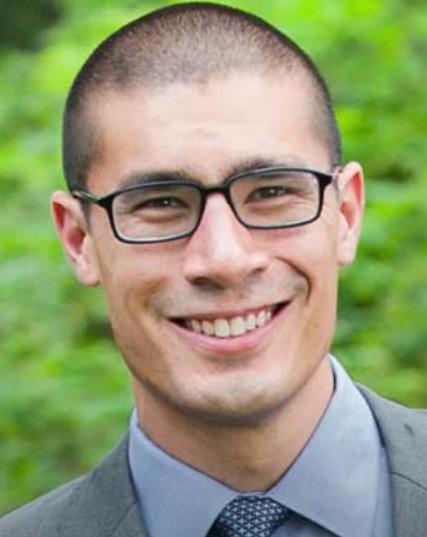}}]
{Michael Posa} (Member, IEEE) received the B.S. and M.S. degrees in mechanical engineering from Stanford University, Stanford, CA, USA, in 2007 and 2008, respectively, and the Ph.D. degree in electrical engineering and computer science from the Massachusetts Institute of Technology, Cambridge, MA, USA, in 2017.,He is currently an Assistant Professor of mechanical engineering and applied mechanics with the University of Pennsylvania, Philadelphia, PA, USA, where he is a Member of the General Robotics, Automation, Sensing and Perception (GRASP) Lab. He holds secondary appointments in electrical and systems engineering and in computer and information science. He leads the Dynamic Autonomy and Intelligent Robotics Lab, University of Pennsylvania, which focuses on developing computationally tractable algorithms to enable robots to operate both dynamically and safely as they maneuver through and interact with their environments, with applications including legged locomotion and manipulation.,Dr. Posa was a recipient of multiple awards, including the NSF CAREER Award, RSS Early Career Spotlight Award, and Best Paper Awards. He is an Associate Editor for IEEE Transactions on Robotics.
\end{IEEEbiography}

\end{document}